\documentclass{article}
\pdfoutput=1
\usepackage[utf8x]{inputenc}
\usepackage[T1]{fontenc}

\usepackage[a4paper,top=3cm,bottom=2cm,left=3cm,right=3cm,marginparwidth=1.75cm]{geometry}

\usepackage{style}

\usepackage{amsfonts, amsmath, amssymb, amsthm}
\usepackage{xfrac}
\usepackage{graphicx, xcolor, colortbl}
\usepackage{footnote}
\usepackage{tablefootnote}
\usepackage{natbib}
\usepackage{dsfont}
\usepackage{bbm}
\usepackage{algorithm}
\usepackage{algpseudocode}
\usepackage{import}
\usepackage{mathtools}
\usepackage{bm}
\usepackage{xspace}
\usepackage{thmtools,thm-restate} 
\usepackage{accents}
\usepackage{framed}
\usepackage{enumitem}
\usepackage{booktabs}

\usepackage{scrextend}

\usepackage{multirow}

\newtheorem{lemma}{Lemma}
\newtheorem{theorem}{Theorem}

\theoremstyle{definition}

\newtheorem{definition}{Definition}

\newtheorem{remark}{Remark}

\usepackage{sty/notations}

\usepackage{todonotes}

\usepackage{minitoc}

\title{\bf \LARGE Fast active learning for pure exploration \\ in reinforcement learning}

\author{\name Pierre M\'enard \email pierre.menard@inria.fr \\
	\addr Inria Lille, SequeL team
	\AND
	\name Omar Darwiche Domingues \email omar.darwiche-domingues@inria.fr \\
	\addr Inria Lille, SequeL team
	\AND
	\name Anders Jonsson \email anders.jonsson@upf.edu \\
	\addr Universitat Pompeu Fabra
	\AND
	\name Emilie Kaufmann \email emilie.kaufmann@univ-lille.fr \\
	\addr CNRS \& ULille (CRIStAL), Inria Lille, SequeL team
	\AND
	\name \'Edouard Leurent \email edouard.leurent@inria.fr \\
	\addr Renault \& Inria Lille, SequeL team
	\AND
	\name Michal Valko \email valkom@deepmind.com \\
	\addr DeepMind Paris
}

\begin{document}

\maketitle

\doparttoc 
\faketableofcontents 

\begin{abstract}\noindent
Realistic environments often provide agents with very limited feedback.
When the environment is initially unknown, the feedback, in the beginning, can be completely absent,
and the agents may first choose to devote all their effort on \emph{exploring efficiently.}
The exploration remains a challenge while it has been addressed with many hand-tuned heuristics with different levels
of generality on one side, and a few theoretically-backed exploration strategies on the other.
Many of them are incarnated by \emph{intrinsic motivation} and in particular \emph{explorations bonuses}.
A common rule of thumb for exploration bonuses is to use $1/\sqrt{n}$ bonus that is added to the empirical estimates of the reward,
where $n$ is a number of times this particular state (or a state-action pair) was visited.
We show that, surprisingly, for a pure-exploration objective of \emph{reward-free exploration},
bonuses that scale with $1/n$ bring faster learning rates, improving the known upper bounds with respect to the dependence on the horizon $H$.
Furthermore, we show that with an improved analysis of the stopping time, we can improve by a factor $H$ the sample complexity
in the \emph{best-policy identification} setting, which is another pure-exploration objective,
where the environment provides rewards but the agent is not penalized for its behavior during the
exploration phase.
\end{abstract}

\section{Introduction}
\label{sec:intro}

In reinforcement learning (RL), an agent learns how to act by interacting with an environment, which provides feedback in the form of reward signals. The agent's objective is to maximize the sum of rewards. In this work, we study how to explore efficiently. In particular  we wish to compute near-optimal policies using the least possible amount of interactions with the environment (in the form of observed transitions). In general, we may be either interested in the performance of the agent during the learning phase or we may only care for the performance of the learned policy. In the first setting, we can measure the performance of the agent by its \emph{cumulative regret} which is the difference between the total reward collected by an optimal policy and the total reward collected by the agent during the learning. Therefore, the agent is encouraged to \emph{explore} new policies but also \emph{exploit} its current knowledge  \citep{bartlett2009regal,jaksch2010near}. Another performance measure related to the regret consists in counting the number of times during the learning that the value of the policy used by the agent is $\epsilon$ far from the optimal one. The minimization of this count is formalized in the PAC-MDP setting introduced by \citet{kakade2013on}, see also \citet{dann2015sample} and \citet{dann2017unifying}.
The second setting and our central focus in this paper is called \emph{pure-exploration} where the agent is free to make mistakes during the learning and explore more vigorously \citep{fiechter1994efficient, kearns1998finite-sample,even2006action}. We provide results for  two pure-exploration settings when the environment is an
episodic \emph{Markov decision process} (MDP): the \emph{reward-free exploration} (RFE) and the \emph{best-policy identification} (BPI).

\paragraph{Best-policy identification} In BPI, an agent interacts with the MDP, observing \emph{transitions} and \emph{rewards}, to output an $\epsilon$-optimal policy with probability at least $1-\delta$ \citep{fiechter1994efficient}.
Most of the work on  BPI assumes that the agent has access to a \emph{generative model} (oracle, \citealp{kearns1998finite-sample}). Having an oracle access means that the agent  can simulate a transition from any sate-action pair. In particular, \citet{azar2013minimax} show that the optimal rate of the sample complexity, defined in this case as the number $n$ of oracle calls for getting an $\epsilon$-optimal policy with probability at least $1-\delta$ is of order\footref{note:empirical-qvi}
$\tcO\left(H^4SA \log(1/\delta)/\epsilon^2\right)$ where $S$ is the size of the state space, $A$ is the size of the action space, and $H$ is the horizon (see Table~\ref{tab:related_work} and also \citealp{agarwal2020model}, \citealp{sidford2018near}). The $\tcO$ notation hides terms that are poly-log in $H,S,A,\epsilon,$ and $\log(1/\delta)$.

 Even if the oracle access is reasonable in some situations (games, physics simulators, \dots), we focus on the  more challenging and practical setting where the agent has only access to a \emph{forward model}, meaning that the agent can only sample \emph{trajectories} from some predefined initial state. In this setting, the sample complexity $\tau$ is the number of trajectories that are necessary to output an $\epsilon$-optimal policy with probability at least $1-\delta$ (which leads to $n = H\tau$ sampled transitions). A straightforward but indirect approach to BPI, suggested for example by \cite{jin2018is}, is to run a regret-minimizing algorithm (for instance, \UCBVI of \citealp{azar2017minimax}) for a sufficiently large number $K$ of episodes, and output a policy $\widehat{\pi}$ that is chosen uniformly at random among the $K$ policies executed by the agent. Unfortunately, this indirect approach is sub-optimal with respect to the error probability $\delta$. Indeed, the resulting sample complexity scales with $1/\delta^2$, instead of the expected $\log(1/\delta)$, as can be seen in Table~\ref{tab:related_work} and Section~\ref{sec:BPI}. Recently, \citet{kaufmann2020adaptive} proposed \RFUCRL, which adapts an episodic version of a \UCRL-type algorithm \citep{jaksch2010near} to best-policy identification. In essence, they replace the random choice of the predicted policy by a data-dependent choice. This algorithm enjoys the correct dependence on $\delta$ prescribed by the lower bound of \citealp{dann2015sample}, but suffers a sub-optimal dependence on $S$, the size of the state space, when $\epsilon$ is small, as well as a sub-optimal dependence on the horizon $H$ (Table~\ref{tab:related_work}).

As an answer to the above sub-optimalities, we propose  \OurAlgorithmBPI, a new algorithm with a sample complexity of  $\tcO\left(SAH^3 \log(1/\delta)/\epsilon^2\right)$, which is optimal in terms of $S, A, \epsilon,$ and $\delta$, according to the lower bound of \citet{dann2015sample}. Moreover, we believe that the dependence on $H$ cannot be improved when the transitions are non-stationary, see the discussion below.
 \OurAlgorithmBPI  is based on \UCBVI  of \citet{azar2017minimax}. It relies on a non-trivial upper bound on the \emph{simple regret} of a \UCBVI-type algorithm (Lemma~\ref{lem:control_gap_BPI}), similar as \citet{dann2019policy}, that shaves the extra $S$ factor of \RFUCRL while keeping the right dependence on $\delta$. The main feature of this upper bound is that it can be computed in the empirical MDP and therefore is accessible to the agent.

\paragraph{Reward-free exploration} Efficient exploration is especially difficult when the reward signals are sparse, as the agent needs to interact with the environment while receiving almost no feedback. To address such situations, we also study \emph{reward-free exploration} introduced by \citet{jin2020reward-free}, where the interaction with the environment is split into two phases: (i) an \emph{exploration} phase, in which the agent learns the transition model $\widehat{p}$ of the MDP by interacting with the environment for a given number of episodes (still with a forward model); and (ii) a \emph{planning} phase, in which the agent receives a reward function $r$ and computes the optimal policy for the MDP parameterized by $(r, \widehat{p}\,)$. Given an accuracy parameter $\epsilon$, we measure the performance of the agent by the number of trajectories required to compute a policy in the planning phase, that is $\epsilon$-optimal \emph{for any given reward function} $r$ with probability at least $1-\delta$.

Our interest in RFE has two major reasons. First, in some applications, it is necessary to compute good policies for a wide range of reward functions. In such case, RFE allows to satisfy this need with only a single exploration phase. Second, RFE gives us good strategies for exploring the environment especially when the reward signal is very sparse or unknown.

One approach to pure exploration is to rely on known \emph{cumulative-regret minimization} methods and their guarantees. This path is taken by  \RFExplore of \citet{jin2020reward-free}. More precisely, \RFExplore  builds upon the \EULER algorithm by \cite{zanette2019tighter} by running one instance of this algorithm for each state $s$ and each episode step $h$ with a reward function incentivizing the visit of state $s$ in step $h$. The leading term in their sample complexity bound scales with   $\tcO\left(S^2 A H^5\log(1/\delta)/\epsilon^2\right)$ for MDPs with $S$ states, $A$ actions, and horizon $H$, which is sub-optimal in $H$ (Table~\ref{tab:related_work}).
\cite{kaufmann2020adaptive} propose \RFUCRL, an alternative algorithm that is reminiscent of the original algorithm proposed by \citet{fiechter1994efficient} for BPI with an improved sample complexity of $\tcO\left(S A H^4(\log(1/\delta)+S)/\epsilon^2\right)$. The main idea behind the algorithm of \citet{fiechter1994efficient} is to build upper confidence bounds on the estimation error of the value function of \emph{any} policy under \emph{any} reward function, and then act greedily with respect to these upper bounds to minimize the estimation error. Using a similar approach, \cite{wang2020on} study reward-free exploration with a particular linear function approximation, providing an algorithm with a sample complexity of order $d^3 H^6\log(1/\delta)/\epsilon^2$, where $d$ is the dimension of the feature space. Finally, \cite{zhang2020task-agnostic} study a setting in which there are only $N$ possible reward functions in the planning phase, for which they provide an algorithm whose sample complexity is $\tcO\left(H^5 S A\log(N)\log(1/\delta)/\epsilon^2\right)$.

In this work, we present \OurAlgorithmRF with sample complexity of
${\tcO \left(S A H^3(\log(1/\delta)+S)/\epsilon^2\right)}$, which improves the bound of \cite{kaufmann2020adaptive} by a factor of $H$. As discussed below, we believe that \OurAlgorithmRF has an \emph{optimal} dependence on $H$ for general, non-stationary transitions.

 A standard path to get such improved dependence is via confidence bonuses built using an \emph{empirical Bernstein inequality} \citep{azar2012on,azar2017minimax,zanette2019tighter}  to make appear a variance term and then to sharply upper-bound these variance terms with a \emph{Bellman type equation for the variances} (see Appendix~\ref{app:Bellman_variance} or \citealp{azar2012on}). However, this standard path  is far less clear for RFE  as the agent does not observe the rewards and therefore cannot compute the empirical variance of the values!
Therefore, one of our main technical contributions is to tackle this challenge by introducing a \emph{new empirical Bernstein inequality} derived from a control of the transition probabilities (Appendix~\ref{app:Bernstein}) and applying the Bellman-type equation for the variances to construct exploration bonuses that do not require a computation of empirical variances. Surprisingly, the bonuses used in \OurAlgorithmRF scale with $1/n(s,a)$ where $n(s,a)$ is the number of times the state-action pair $(s,a)$ was visited, instead of the usual $1/\sqrt{n(s,a)}$ bonus.

\begin{remark}[{\bf on the optimality of \OurAlgorithmBPI \& \OurAlgorithmRF}]
	\label{remark:optimality}
	As summarized in Table \ref{tab:related_work}, the algorithms we propose for BPI and RFE are sub-optimal \emph{only by a factor of $H$}. However, we would like to highlight that the lower bounds in both settings are proved for \emph{stationary transitions}, that is, when the transition probabilities are the same in every stage of the episode, whereas our upper bounds hold for \emph{non-stationary transitions}, i.e., when these probabilities can be different in each stage. In the non-stationary setting, we might expect these lower bounds to have an extra factor of $H,$\footnote{\cite{jin2018is} prove that this is indeed the case for \emph{regret} lower bounds.} in which case our algorithms would be optimal. Notice that in Table \ref{tab:related_work}, the upper bounds are for the non-stationary case, whereas the available lower bounds are for the stationary case.
\end{remark}

\paragraph{Contributions} To sum up, we highlight our major contributions.

\begin{itemize}
\item BPI: we provide \OurAlgorithmBPI,  with a sample complexity of $\tcO\left(H^3SA\log(1/\delta) \right)$ when $\epsilon$ is small enough. Up to a factor of $H$ and poly-log terms, it matches  the lower bound of \citet{dann2015sample} and improves the dependence either on $H$, $1/\delta$ or $S$ with respect to previous work.
\item RFE:  we provide \OurAlgorithmRF with a sample complexity of  ${\tcO\left(H^3SA(\log(1/\delta)+S)/\varepsilon^2\right)}$. Up to a factor of $H$ and poly-log terms, our rate matches simultaneously the lower bound $\Omega(H^2S^2A/\epsilon^2)$ by \citet{jin2020reward-free}, effective when $\delta$ is fixed and $\epsilon$ goes to zero, and the lower bound $\Omega(H^2SA \log(1/\delta)/\epsilon^2)$ by \citet{dann2015sample}, effective when $\epsilon$ is fixed and $\delta$ goes to zero.
 \item Due to the absence of the rewards in RFE, known techniques to get the optimal dependence in the horizon $H$ \citep{azar2012on,zanette2019tighter} do not apply. We therefore develop a new analysis that relies on the use of exploration bonuses scaling with $1/n$ instead of the standard $1/\sqrt{n}$.
\end{itemize}

\begin{table}[h]
	\centering
	{\footnotesize
	\renewcommand{\arraystretch}{2}

	\begin{tabular}{@{}llll@{}}
		\toprule
 {\bf Algorithm}           	    & {\bf Setting}  &  \vtop{\hbox{\strut {\bf Upper bound}}\hbox{\strut (\textit{non-stationary case})}} & \vtop{\hbox{\strut {\bf Lower bound}}\hbox{\strut (\textit{stationary case})}} \\
 \hline\hline
 Empirical \QVI \citep{azar2012on}%
 \tablefootnote{\label{note:empirical-qvi}\citet{azar2012on}, express both the upper and the lower bounds in the \textbf{total number of calls to the generative model (instead of trajectories)} and prove them for the $\gamma$-discounted infinite horizon. They are of order $SA(1-\gamma)^{-3}\epsilon^{-2}\log(1/\delta)$. We translate them to the episodic setting by replacing $(1-\gamma)^{-1}$ by the horizon $H$.
 In particular, the term $1/(1-\gamma)^3$ translates to $H^3$ and we include \textbf{an extra $H$ factor in the upper bound} due to the non-stationary transitions, i.e., when the transition probabilities depend on the stage $h \in [H]$.
 }
  &   gen.\,model   &  $\frac{H^4SA}{\epsilon^2}\log\left(\frac{1}{\delta}\right)$         & $\frac{H^3SA}{\epsilon^2}\log\left(\frac{1}{\delta}\right)$  \\\midrule
 \UCBVI+\;random recomm.\tablefootnote{See the last paragraph of Section~\ref{ss:bpiucrl} for a discussion.}
  & BPI & $\frac{H^3 SA}{\epsilon^2} \frac{\log(1/\delta)}{\delta^2}$ &
 \multirow{3}{*}{$\frac{H^2SA}{\epsilon^2}\log\left(\frac{1}{\delta}\right)$ }
 \\
 \BPIUCRL \citep{kaufmann2020adaptive}   & BPI        &   $\frac{H^4SA}{\epsilon^2}\left( \log\left(\frac{1}{\delta}\right) + S\right)$          &             \\
\rowcolor{Gray}
\OurAlgorithmBPI \emph{(this work)}	& BPI   & $\frac{H^3SA}{\epsilon^2}\log\left(\frac{1}{\delta}\right)$            &     \\\midrule

\UCBZero \citep{zhang2020task-agnostic}	& RFE, $N$ tasks & $\frac{H^5 SA\log(N)}{\epsilon^2}\log\left(\frac{1}{\delta}\right)$            &   $\frac{H^2 SA\log(N)}{\epsilon^2}$           \\ \midrule

\RFExplore \citep{jin2020reward-free}	& RFE    &  $\frac{H^7S^2A}{\epsilon}\log^3(\frac{1}{\delta}) +\frac{H^5S^2A}{\epsilon^2}\log\left(\frac{1}{\delta}\right)$           &
\multirow{3}{*}{ $\frac{H^2 SA}{\epsilon^2}\left(\log\left(\frac{1}{\delta}\right) + S \right)$\tablefootnote{We combined the $\Omega\left(H^2S^2/\epsilon^2\right)$ result of
	\citet{jin2020reward-free} and the $\Omega\left(H^2SA\log(1/\delta)/\epsilon^2\right)$ result of \citet{dann2015sample}.} }
       \\
\RFUCRL	\citep{kaufmann2020adaptive} &   RFE                                       &   $\frac{H^4SA}{\epsilon^2}\left( \log\left(\frac{1}{\delta}\right) + S \right)$           &             \\
\rowcolor{Gray} \OurAlgorithmRF \emph{(this work)}				&  RFE                                           &   $\frac{H^3 SA}{\epsilon^2}\left(\log\left(\frac{1}{\delta}\right) + S \right)$          &          \\ \bottomrule
	\end{tabular}
	}

	\caption{Best-policy identification (BPI) and reward-free exploration (RFE) algorithms with their respective upper bounds on the sample complexity, expressed in terms of the number of trajectories required by the algorithms.\footref{note:empirical-qvi}
	The factors and terms that are poly-log in $S,A,H,\epsilon,$ and $\log(1/\delta)$ are omitted.
\label{tab:related_work}
  }

\end{table}

\section{Setting}
\label{sec:setting}

 We consider a finite episodic MDP $\left(\cS, \cA, H, \{p_h\}_{h\in[H]},\{r_h\}_{h\in[H]}\right)$,  where $\cS$ is the set of states, $\cA$ is the set of actions, $H$ is the number of steps in one episode, $p_h(s'|s,a)$ is the probability transition from state~$s$ to state~$s'$ by taking the action $a$ at step $h,$ and $r_h(s,a)\in[0,1]$ is the bounded deterministic reward received after taking the action $a$ in state $s$ at step $h$. Note that we consider the general case of rewards and transition functions that are possibly non-stationary, i.e., that are allowed to depend on the decision step $h$ in the episode. We denote by $S$ and $A$ the number of states and actions, respectively.

\paragraph{Learning problem} The agent, to which the transitions are \emph{unknown}, interacts with the environment in episodes of length $H$, with a \emph{fixed} initial state $s_1$.\footnote{As explained by \citet{fiechter1994efficient} and \citet{kaufmann2020adaptive}, if the first state is sampled randomly as $s_1\sim p_0,$ we can simply add an artificial first state $s_0$ such that for  any action $a$, the transition probability is defined as the distribution $p_0(s_0,a) \triangleq p_0.$} At each step $h\in[H]$, the agent observes a state $s_h\in\cS$, takes an action $a_h\in\cA$ and  makes a transition to a new state $s_{h+1}$ according to the probability distribution $p_h(s_h,a_h)$. In  BPI, the agent receives a deterministic reward $r_h(s_h,a_h)$ at each step $h$, for \emph{fixed} reward functions $r \triangleq \{r_h\}_{h\in[H]}$, and it is required to output an $\epsilon$-optimal policy \emph{with respect to $r$}. In RFE, no rewards are observed during exploration, and the agent is required to output an estimate of the transition probabilities which can be used afterwards to compute an $\epsilon$-optimal policy for \emph{any reward function}.

\paragraph{Policy \& value functions} A \emph{deterministic} policy $\pi$ is a collection of functions $\pi_h : \cS \mapsto \cA$ for all $h\in [H]$, where every $\pi_h$  maps each state to a \emph{single} action. The value functions of $\pi$, denoted by $V_h^\pi$, are defined as 
\begin{align*}
	V_h^\pi(s) \triangleq \E\left[ \sum_{h'=h}^H r_{h'}(s_{h'}, a_{h'}) \,\middle|\, s_h = s  \right], 
	\quad \text{where} \quad a_{h'} \triangleq \pi_{h'}(s_{h'}) \quad \text{and} \quad s_{h'+1}\sim p_{h'}(s_{h'}, a_{h'})
\end{align*}
for $h \in [H]$. The optimal value functions are defined as $V_h^\star(s) \triangleq \max_{\pi} V_h^\pi(s)$. Both $V_h^\pi$ and $V_h^\star$ satisfy the Bellman equations \citep{puterman1994markov}, that are expressed using the Q-value functions $Q_h$ and $Q_h^\star$ in the following way,
\begin{align*}
	 V_h^\pi(s) &= \pi_h Q_h^\pi (s),
	\; \text{ with } \quad 
	Q_h^{\pi}(s,a) \triangleq r_h(s,a) + p_h V_{h+1}^\pi(s,a), \quad \text{and}
\\
 V_h^\star(s) &= \max_a Q_h^\star (s, a),
\; \text{ with } \quad 
Q_h^\star(s,a)\triangleq r_h(s,a) + p_h V_{h+1}^\star(s,a),
\end{align*}
where by definition, $V_{H+1}^\star \triangleq V_{H+1}^\pi \triangleq 0$. Furthermore, $p_{h} f(s, a) \triangleq \E_{s' \sim p_h(\cdot | s, a)} \left[f(s')\right]$   denotes the expectation operator with respect to the transition probabilities $p_h$ and
$(\pi_h g)(s)  \triangleq \pi_h g(s) \triangleq  g(s,\pi_h(s))$ denotes the composition with the policy~$\pi$ at step $h$.

\paragraph{Empirical MDP} Let $(s_h^i, a_h^i, s_{h+1}^i)$ be the state, the action, and the next state observed by an algorithm at step $h$ of episode $i$.  For any step $h\in [H]$ and any state-action pair $(s,a) \in \cS \times \cA$, we let $n_h^{t}(s,a) \triangleq \sum_{i=1}^{t} \ind{\left\{(s_h^i,a_h^i) = (s,a)\right\}}$ be the number of times the state action-pair $(s,a)$ was visited in step $h$ in the first $t$ episodes and  $n_h^{t}(s,a,s') \triangleq \sum_{i=1}^{t} \ind{\big\{(s_h^i,a_h^i,s_{h+1}^i) = (s,a,s')\big\}}$. These definitions permit us to define the empirical transitions as
\[\hp^{\,t}_h(s' | s,a) \triangleq \frac{n_h^t(s,a,s')}{n_h^t(s,a)} \text{ if } n_h^t(s,a)>0 \quad \text{and} \quad  \hp^{\,t}_h(s' | s,a) \triangleq \frac{1}{S} \quad \text{otherwise}. \]
Based on the empirical transitions and on exploration bonuses, we  introduce various data-dependent quantities that are useful for designing algorithms for either the BPI or the RFE objective. While the former allows the agent to access the reward function $r$ during exploration, the latter does not. Therefore, in all data-dependent quantities introduced in Section~\ref{sec:reward_free} to design a RFE algorithm, we  always materialize a possible dependency on $r$. In particular, we denote by $\hV^{t,\pi}_h(s ; r)$ and $\hQ^{t,\pi}_h(s ,a ; r)$ the value  and the action-value functions of a policy $\pi$ in the MDP with transition kernels $\hp^{\,t}$ and reward function $r$. In Section~\ref{sec:BPI}, in which the reward function is fixed, we  drop the dependency on $r$ and use the simpler notation 
$\hV^{t,\pi}_h(s )$ and $\hQ^{t,\pi}_h(s ,a )$.
\section{Reward-free exploration}
\label{sec:reward_free}

In this section, we consider reward-free exploration (RFE)  where the agent \emph{does not observe the rewards} during the exploration phase. 
Again, as the value functions defined in Section~\ref{sec:setting} depend on a reward function $r$, we sometimes use the notation $V_h(s ; r)$ and $Q_h (s,a ; r)$ instead of $V_h(s)$ and $Q_h(s,a)$.

\paragraph{Reward-free exploration} In RFE , the agent interacts with the MDP in the following way. At the beginning of the episode $t$, the agent decides to follow a policy~$\pi^t$, called the sampling rule, based only on the data collected up to episode $t-1$. Then, a \emph{reward-free} episode $z^t \triangleq (s_1^{t},a_1^{t},s_2^{t},a_2^{t},\dots,s_H^{t},a_H^{t})$  is generated starting from the the initial state $s_1^t\triangleq s_1$ by taking actions $a_{h}^t = \pi_{h}^t(s_{h}^t)$ and, for $h>1$, observing next-states according to $s_h^t \sim p_h(s_{h-1}^{t}, a_{h-1}^{t})$.
This new trajectory is added to the dataset $\cD_{t} \triangleq \cD_{t-1} \cup \{z_{t}\}$.
At the end of each episode, the agent can decide to stop collecting data, according to a \emph{random} stopping time $\tau$ and outputs an empirical transition kernel $\hp$ built with the dataset $\cD_\tau$.

Any RFE agent is therefore made of a triple $((\pi^{t})_{t\in\N^\star},\tau,\hp\,)$. Our goal is to design an agent that is  $(\varepsilon,\delta)$-PAC, \emph{probably approximately correct}, according to the following definition, for which the number of exploration episodes $\tau$, i.e., the \emph{sample complexity}, is as small as possible.

\begin{definition}[PAC algorithm for RFE] \label{def:RFE} An algorithm is $(\varepsilon,\delta)$-PAC for reward-free exploration if
\[\P\left(\text{for any reward function\ } r, V_1^{\star}(s_1 ; r) - V_1^{\hpi^{\star}_r}(s_1 ; r) \leq \epsilon\right) \geq 1 - \delta,\]
where $\hpi^{\star}_r$ is the optimal policy in the empirical MDP whose transitions are given by the transition kernel $\hp$ returned by the algorithm and whose reward function is $r$.

\end{definition}

\subsection{\OurAlgorithmRF}

In this section, we present  the \OurAlgorithmRF algorithm along with a high-probability bound on its sample complexity.  \OurAlgorithmRF relies on upper bounds on the estimation error between the true value functions and their empirical counterparts.
We start with the motivation for the design choices for \OurAlgorithmRF.  We then introduce quantities to which we engrave our choices
and which we subsequently use in the definition algorithm.  In the algorithmic template
we proceed as \cite{fiechter1994efficient} and \citet{kaufmann2020adaptive}  by upper bounding the estimation-error for all the policies \emph{with the striking difference that we only upper bound it at the initial state}. We finish this part by providing more intuition and discussion, in particular, we provide
\emph{technical insights into what \OurAlgorithmRF is optimizing} and then \emph{explain our $1/n$ versus $1/\sqrt{n}$ exploration bonuses},
the reasons for choosing them and the challenge with analysing them.

\paragraph{Estimation error} Given a policy $\pi$ and an arbitrary reward function $r$, we define the estimation error as the absolute difference between the Q-value of $\pi$ in the empirical MDP and its Q-value in the true MDP. Precisely, after episode $t$, for all $(s,a,h)$,  we define
\begin{align*}
  \he_h^{\,t,\pi} (s,a ; r) \triangleq \big|\hQ_h^{t,\pi}(s,a ; r) - Q_h^\pi(s,a ; r)\big|.
\end{align*}
To control the approximation error of the value of \emph{any policy} for \emph{any reward function} starting from the initial state $s_1$, we introduce the functions $W_h^t(s,a)$ defined inductively by  $W_{H+1}^{t}(s,a) \triangleq 0$ for all $(s,a) \in \cS \times \cA$ and for all $h\in[H]$ and $(s,a) \in \cS\times \cA$,
\begin{equation}\label{eq:def_W}
W_h^{t}(s,a) \triangleq \min\left(\!H, 15H^2 \frac{\beta(n_h^t(s,a),\delta)}{n_h^t(s,a)} +\left(1+\frac{1}{H}\right)\sum_{s'}\hp_h^{\;t}(s'|s,a) \max_{a'} W_{h+1}^{t}(s',a')\!\right),
\end{equation}
where $\beta(n_h^t(s,a),\delta)$ is a threshold that depends on how we build the confidence sets for the transitions probabilities. Notice that the $W_h^t$ are all \emph{independent} of the reward function $r$. As shown in the next lemma with  proof in Appendix~\ref{app:proofs_RF}, the function $W_1^t(s_1,a)$ can be used to upper bound the estimation error of any policy under any reward function in the initial state $s_1$.

\begin{restatable}{lemma}{lemmaWRF}
  \label{lem:upper_bound_error_W} With probability at least $1-\delta$, for any episode $t$, policy $\pi$, and  reward function $r$,
  \[
  \he_1^{\,t,\pi}(s_1, \pi_1(s_1) ; r) \leq  3e \sqrt{\max_{a\in\cA} W_1^{t} (s_1,a)} + \max_{a\in\cA} W_1^{t}(s_1,a).
  \]
In particular, the above holds on the event $\cF$ defined in Appendix~\ref{app:concentration}.
\end{restatable}
\noindent
With all the above definitions, we are now ready to outline our \OurAlgorithmRF algorithm.
\begin{framed}
\begin{itemize}[leftmargin=9em]
 \item[\bf  sampling rule:] the policy $\pi^{t+1}$ is the greedy policy with respect to $W_h^{t},$
  \[\forall s \in \cS, \forall h \in [H], \ \ \pi^{t+1}_h(s) = \argmax_{a\in\cA} W_h^{t}(s,a)\]
  \vspace{-2em}
 \item[\bf  stopping rule:] $\tau  = \inf \left\{ t\in \N : 3e \sqrt{\pi_1^{t+1} W_1^{t} (s_1)} + \pi_1^{t+1} W_1^{t}(s_1)  \leq \varepsilon/2 \right\}$
 \item[\bf  prediction rule:] output the empirical transition kernel $\hp = \hp^{\,\tau}$
\end{itemize}
\end{framed}
\noindent
Next, we provide a bound on the sample complexity of \OurAlgorithmRF with a proof in Appendix~\ref{app:proofs_RF}.
\begin{restatable}{theorem}{theoremRF}\label{thm:sample_complexity_RF} For $\delta\in(0,1)$, $\epsilon\in (0,1]$, \OurAlgorithmRF with threshold $\beta(n,\delta) \triangleq \log\left(3SAH/\delta\right) + S\log\left(8e(n+1)\right)$
is $(\epsilon,\delta)$-PAC for reward-free exploration. Moreover, \OurAlgorithmRF stops after $\tau$ steps where, with probability  at least $1-\delta$,
\begin{align*}
 \tau \leq \frac{H^3SA}{\epsilon^2} \left( \log(3SAH/\delta) +S\right) C_1+1
\end{align*}
and where $C_1 \triangleq 5587 e^6 \log\left(e^{18} \left( \log(3SAH/\delta) +S\right) H^3SA/\epsilon\right)^2$.
\end{restatable}
\noindent
As a consequence, the sample complexity of \OurAlgorithmRF is of order $\tcO\left(H^3SA(\log(1/\delta)+S)\right)$ and matches the lower bound of $\Omega(H^2S^2A/\epsilon^2)$ of \citet{jin2020reward-free} up to a factor of $H$ and poly-log terms. This lower bound is informative in the regime where $\delta$ is considered as fixed and $\epsilon$ tends to zero. Moreover, up to a factor $H$, our result
also matches the lower bound of $\Omega\left(H^2SA\log(1/\delta)/\epsilon^2\right)$ given by \citet{dann2015sample} which is informative in the regime where $\epsilon$ is fixed and $\delta$ tends to $0$.  In Remark~\ref{remark:optimality}, we discuss the optimality of \OurAlgorithmRF with respect to $H$. As we see in the next section, the quadratic dependence on $S$ can be avoided for BPI.

\paragraph{What is \OurAlgorithmRF optimizing?} Contrary to \RFUCRL of \citet{kaufmann2020adaptive}, \OurAlgorithmRF does \emph{not} build upper bounds on all estimation errors $\he_h^{\,t,\pi}(s,a ; r)$ for all $h\in [H]$ but \emph{only for the one at the initial state} $\he_1^{\,t,\pi}(s_1,\pi_1(s_1) ; r)$. Moreover, the upper bound is not $W_1^t(s_1,a)$ itself, but a function of this quantity, as can be seen in Lemma~\ref{lem:upper_bound_error_W}. Hence, if \OurAlgorithmRF actually follows the optimism-in-the-face-of-uncertainty principle, what quantities are  $W_h^t$ upper bounding?
To answer this question and provide an intuition on the sampling rule of \OurAlgorithmRF,
fix a policy $\pi$ and let $P^\pi$ be the probability distribution governing a random trajectory $(s_1,a_1,s_1,a_2, \ldots,s_H,a_H)\sim P^\pi$ in the MDP. Next, let  $\hP^{t,\pi}$ be the probability distribution of a trajectory $(s_1^t,a_1^t,s_1^t,a_2^t, \ldots,s_H^t,a_H^t)\sim \hP^{t,\pi}$ in the empirical MDP built using the dataset $\cD_t$ at episode $t$. Assuming that all the state action pairs have been visited at least once at time $t$, using the  chain rule (see~\citealp{garivier2019explore}) we can compute the Kullback-Leibler divergence between these two probability distributions as
\[
  \KL(\hP^{t,\pi},P^\pi) = \sum_{h=1}^H \sum_{s,a} \hp_h^{\,t,\pi}(s,a) \KL\left(\hp_h^{\,t,\pi}(s,a),p_h(s,a)\right),
\]
where $\hp_h^{\;t,\pi}(s,a)$ is the probability to reach state-action $(s,a)$ at step $h$ under policy $\pi$ in the empirical MDP in episode $t$.
Notice now that the bonus of the form $\beta(n,\delta)/n$ used to define $W^t$ is  \emph{by design chosen to be an upper-confidence bound} on the Kullback-Leibler divergence between the empirical transition probability and the transition probability. Indeed, in Appendix~\ref{app:concentration} we show that with high probability, for all $(s,a) \in \cS \times \cA$ and $h \in [H]$,
\[
\KL\left(\hp_h^{\,t}(s,a),p_h(s,a)\right) \leq \frac{\beta(n_h^t(s,a),\delta)}{n_h^t(s,a)}\cdot
\]
Therefore, omitting the clipping to $H$ in~\eqref{eq:def_W}, we have that
\[
\max_{\pi}\KL(\hP^{t,\pi},P^\pi) \lesssim \frac{ \pi_1^{t+1} W_1^t(s_1)}{H^2}.
\]
Therefore, \OurAlgorithmRF can be interpreted as an algorithm minimizing an upper-confidence bound on the loss of  $\max_{\pi}\KL(\hP^{t,\pi},P^\pi)$, which requires bonuses of the form $\beta(n,\delta)/n$ instead of $\sqrt{\beta(n,\delta)/n}$. Notice that this loss is of the same flavor as the one introduced by \citet{hazan2019provably}.

\paragraph{Bonuses of $1/n$ versus $1/\sqrt{n}$} Our approach differs from the bonuses typically used in regret minimization (e.g., \citealp{azar2017minimax}) and in prior work in reward-free exploration \citep{kaufmann2020adaptive,zhang2020task-agnostic}, which uses bonuses proportional to $\sqrt{1/n(s,a)}$.  Intuitively, since $1/n$-bonuses decay faster with $n$, our algorithm is more exploratory: once a state-action pair $(s,a)$ has been visited, the bonus associated to it will be more strongly reduced than if we used $\sqrt{1/n}$-bonuses and the algorithm tends to visit other state-action pairs before returning to $(s, a)$ again. Technically, this might seem very surprising. Indeed, if we want to estimate the mean $\mu$ of a random variable $X$ with an estimator $\widehat{\mu}_n$ computed with $n$ i.i.d.\,samples from $X$, the error $\abs{\mu-\widehat{\mu}_n}$ scales with $\sqrt{1/n}$ by Hoeffding's inequality, which explains the shape of the bonuses used in previous works.
However, instead of bounding the error $\abs{\mu-\widehat{\mu}_n}$, our concentration inequalities based on the KL divergence give us a bound on the quadratic term $(\mu-\widehat{\mu}_n)^2\!\!,$ which scales with~$1/n$. This allows us to use a Bellman-type equation for the variance of the value functions and reduce the sample complexity by a factor of $H$, similarly to previous work on regret minimization \citep{azar2017minimax}. The main challenge in our case is that, in reward free exploration, we need to upper bound the sum of variances for \emph{any possible value function}, which makes this technique considerably more challenging to analyze than for the regret minimization.

\subsection{Proof sketch}

We first sketch the proof of Lemma~\ref{lem:upper_bound_error_W}. We begin as it is done in the analysis of \citet{kaufmann2020adaptive}. For a fixed policy $\pi$ and an arbitrary reward function $r$, we  decompose the estimation error of the Q-value function of $\pi$ at the state-action pair $(s,a)$ as, for all reward function $r$,
\begin{align*}
\he_h^{\,t,\pi}(s,a ; r)&\leq \big|\hQ_h^{\,t,\pi}(s,a ; r)-Q_h^\pi(s,a; r)\big| \leq \big|(\hp_h^{\,t}-p_h)V_{h+1}^{\pi}(s,a)\big|+\hp_h^{\,t}|\hV^{t,\pi}_{h+1}-V_{h+1}^\pi|(s,a)\\
&=\big|(\hp_h^{\,t}-p_h)V_{h+1}^{\pi}(s,a)\big|+ \hp_h^{\,t} \pi_{h+1}^{\,t} \he_{h+1}^{\,t,\pi}(s,a ; r).
\end{align*}
Similarly to \citet{azar2017minimax} and \citet{zanette2019tighter}, to obtain the optimal dependency with respect to the horizon $H$, we would like to apply the Bernstein inequality to control the first term. Since we need to do it for all value functions of all policies, we could use a covering of this function space and conclude with a union bound, see \citet{domingues2020regret}. Instead we show, via Lemma~\ref{lem:Bernstein_via_kl} in Appendix~\ref{app:Bernstein}, that from a control of the deviations of the empirical transition probabilities such that
\[
\KL\left(\hp_h^{\,t}(s,a),p_h(s,a)\right) \leq \frac {\beta(n_h^t(s,a), \delta)} {n_h^t(s,a)}
\]
with high probability, we deduce an empirical Bernstein inequality,
\begin{align*}
  \he_h^{\,t,\pi}(s,a ; r)\leq &\ 3\sqrt{ \frac{\Var_{\hp_h^{\,t}}(\hV_{h+1}^{t,\pi})(s,a ; r)}{H^2} \left(\frac{H^2\beta(n_h^t(s,a),\delta)}{n_h^t(s,a)}\wedge 1\right)} +15H^2\frac{\beta(n_h^t(s,a),\delta)}{n_h^t(s,a)}\\
  &+\left(1+\frac{1}{H}\right)\hp_h^{\,t} \hpi_{h+1}^{\,t} \he_{h+1}^{\,t,\pi}(s,a ; r),
\end{align*}
where the variance of $\hV_{h+1}^{t,\pi}$,\emph{ in particular with respect to} $\hp_h^{\,t}(\cdot|s,a)$ is defined as
\begin{align*}
	\Var_{\hp_h^{\,t}}(\hV_{h+1}^{t,\pi})(s,a ; r)
	= \sum_{s'}  \hp_h^{\,t}(s'|s,a)
	\left(
	    \hV_{h+1}^{t,\pi}(s' ; r)
	    -
	    \E_{z \sim \hp_h^{\,t}(\cdot|s,a)}\left[\hV_{h+1}^{t,\pi}(z ; r) \right]
	\right)^{\!2}.
\end{align*}

\noindent Therefore, defining  $Z_{H+1}^{t,\pi}(s,a ; r) \triangleq 0$ and recursively the functions
\begin{align*}
  Z_h^{t,\pi}(s,a ; r) &\triangleq\min\!\Bigg(\!H, 3\sqrt{ \frac{\Var_{\hp_h^{\,t}}(\hV_{h+1}^{t,\pi})(s,a ; r)}{H^2} \left(\frac{H^2\beta(n_h^t(s,a),\delta)}{n_h^t(s,a)}\wedge 1\right)} +15H^2\frac{\beta(n_h^t(s,a),\delta)}{n_h^t(s,a)}\\
 &\qquad\qquad+\left(1+\frac{1}{H}\right)\hp_h^{\,t} \pi_{h+1} Z_{h+1}^{t,\pi}(s,a ; r)\Bigg)\CommaBin\end{align*}
we prove by induction that for all $(s,a,h),$
 \begin{equation}
   \label{eq:he_lower_W_main}
 \he_h^{\,t,\pi}(s,a ; r)\leq Z_h^{t,\pi}(s,a ; r).
\end{equation}
We now \emph{split} $Z^{t,\pi}$ in two terms.  The first term is the one with the bonus in $\sqrt{1/n}$ and the second one with the bonus in $1/n$. Precisely, for all $(s,a),$ we define recursively two other quantities $Y_{H+1}^{t,\pi}(s,a ; r) \triangleq W_{H+1}^{t,\pi}(s,a) \triangleq 0$ and
\begin{align*}
  Y_h^{t,\pi}(s,a ; r) &\triangleq 3\sqrt{ \frac{\Var_{\hp_h^{\,t}}(\hV_{h+1}^{t,\pi})(s,a ; r)}{H^2} \left(\frac{H^2\beta(n_h^t(s,a),\delta)}{n_h^t(s,a)}\wedge 1\right)} + \left(1+\frac{1}{H}\right)\hp_h^{\,t} \pi_{h+1} Y_{h+1}^{t,\pi}(s,a ; r)\\
  W_h^{t,\pi}(s,a) &\triangleq\min\left(\!H, 15H^2 \frac{\beta(n_h^t(s,a),\delta)}{n_h^t(s,a)} +\left(1+\frac{1}{H}\right)\hp_h^{\,t} \pi_{h+1} W_{h+1}^{t,\pi}(s,a) \right).
\end{align*}
We then prove by induction (Appendix~\ref{app:proofs_RF}, Step 2 of the proof of Lemma \ref{lem:upper_bound_error_W}) that for all $h,s,a,$
\begin{align}
  \label{eq:Z_lower_YW_main}
Z_h^{t,\pi}(s,a ; r) \leq Y_h^{t,\pi}(s,a ; r)+W_h^{t,\pi}(s,a).
\end{align}
Note that although $Z_h^{t,\pi}(\cdot ; r)$ is a high-probability upper bound on $\widehat{e}_h^{\,t,\pi}(\cdot ; r)$, we cannot use it to build a sampling rule reducing the errors as it still depends on the reward function $r$ through the empirical variance term, and this knowledge is only available in the planning phase. To obtain an upper bound on $Z_h^{t,\pi}(\cdot ; r)$ which does not depend on $r$, we now further upper-bound $Y^{t,\pi}(\cdot ; r)$. The key tool for this purpose is to use the Bellman equation for the variances, see Appendix~\ref{app:Bellman_variance}. We denote by $\hp_h^{\,t,\pi}(s,a)$ the probability of reaching the state-action pair $(s,a)$ at step $h$ under the policy~$\pi$ in the empirical MDP at time $t$. Using Cauchy-Schwarz inequality, Lemma~\ref{lem:law_of_total_variance} in Appendix~\ref{app:Bellman_variance}, and the fact that that variance of the sum of reward is upper bounded by $H^2$, we get
\begin{align*}
  \pi_1 Y_1^{t,\pi}(s_1; r) &= 2\sum_{s,a}\sum_{h=1}^H \hp_h^{t,\pi}(s,a) \left(1+\frac{1}{H}\right)^{h-1}  \sqrt{ \frac{\Var_{\hp_h^{\,t}}(\hV_{h+1}^{t,\pi})(s,a; r)}{H^2} \left(\frac{H^2\beta(n_h^t(s,a),\delta)}{n_h^t(s,a)}\wedge 1\right)}\\
  &\leq 3e \sqrt{\sum_{s,a}\sum_{h=1}^H \hp_h^{\,t,\pi}(s,a)\frac{ \Var_{\hp_h^{\,t}}(\hV_{h+1}^{t,\pi})(s,a ; r)}{H^2}}\sqrt{\sum_{s,a}\sum_{h=1}^H \hp_h^{\,t,\pi}(s,a) \left(\frac{H^2\beta(n_h^t(s,a),\delta)}{n_h^t(s,a)}\wedge 1\right)}\\
  &\leq 3e \sqrt{\frac{1}{H^2}  \E_{\pi, \hp_h^{\,t}}\!\!\left[ \left(\sum_{h=1}^H r_h(s_h,a_h) -\hV_1^{\pi}(s_1; r)\right)^2\right]  }\sqrt{\sum_{s,a}\sum_{h=1}^H \hp_h^{\,t,\pi}(s,a) \left(\frac{H^2\beta(n_h^t(s,a),\delta)}{n_h^t(s,a)}\wedge 1\right)}\\
  &\leq 3e\sqrt{\sum_{s,a}\sum_{h=1}^H \hp_h^{\,t,\pi}(s,a) \left(\frac{H^2\beta(n_h^t(s,a),\delta)}{n_h^t(s,a)}\wedge 1\right)}\leq 3e\sqrt{W_1^{t,\pi}(s_1)},
\end{align*}
where the last inequality is proved in Appendix~\ref{app:proofs_RF} (Step 3 of the proof of Lemma \ref{lem:upper_bound_error_W}).

\noindent
Combining inequality $\pi_1 Y_1^{t,\pi}(s_1; r) \leq 3e\sqrt{W_1^{t,\pi}(s_1)}$ with~\eqref{eq:he_lower_W_main} and~\eqref{eq:Z_lower_YW_main} yields that, for all $\pi$,
\[
\he_1^{\,t,\pi}(s_1,\pi_1(s_a) ; r) \leq  3e \sqrt{ \pi_1 W_1^{t,\pi} (s_1)} + \pi_1 W_1^{t,\pi}(s_1).
\]
Finally, we note that by construction, $\pi_1 W_1^{t,\pi}(s_1)\leq \max_{a\in\cA} W_1^{t}(s_1,a)$, which allows us to conclude the proof of Lemma~\ref{lem:upper_bound_error_W}.

Next, we sketch the proof of Theorem~\ref{thm:sample_complexity_RF}. The fact that $\OurAlgorithmRF$ is $(\epsilon,\delta)$-PAC is a simple consequence of Lemma~\ref{lem:upper_bound_error_W}. Indeed, on an event of probability at least $1-\delta$, if the algorithm stops at time $\tau$ we know that for all policy $\pi$ and for all reward function $r$,
\begin{align*}
  \frac{\epsilon}{2} &\geq 3e \sqrt{\max_{a\in\cA} W_1^{\tau} (s_1,a)} + \max_{a\in\cA} W_1^{\tau}(s_1,a)
  \geq \he_1^{\,\tau,\pi}(s_1,\pi_1(s_1) ; r) = | \hV^{\tau,\pi}_1(s_1 ; r) - V_1^\pi(s_1 ; r)|.
\end{align*}
Therefore, still on the same event it holds that
\begin{align*}
V_1^{\star}(s_1 ; r) - V_1^{\hpi^{\star}}(s_1 ; r) &= V_1^{\star}(s_1 ; r) - \hV_1^{\tau, \pistar_r}(s_1 ; r) + \hV_1^{\tau,\pistar_r}(s_1 ; r) - \hV_1^{\tau, \hpi^{\star}_r}(s_1 ; r) + \hV_1^{\tau, \hpi^{\star}_r}(s_1 ; r)- V_1^{\hpi^{\star}_r}(s_1 ; r) \\
&\leq |V_1^{\pi^\star_r}(s_1 ; r) - \hV_1^{\tau, \pistar_r}(s_1 ; r)| + |\hV_1^{\tau, \hpi^{\star}_r}(s_1 ; r)- V_1^{\hpi^{\star}_r}(s_1 ; r)| \leq \epsilon\,.
\end{align*}
The proof of the bound on the sample complexity is close to the one of a regret bound. We fix a time $t<\tau$. We start by proving an upper-bound on $W_1^t(s_1,\pi^{t+1}(s_1))$. For that using again the empirical Bernstein inequality of Lemma~\ref{lem:Bernstein_via_kl}, with high probability, it holds that
\begin{equation*}
  W_h^t(s,a) \leq 21 H^2 \left(\frac{\beta(n_h^t(s,a),\delta)}{n_h^t(s,a)}\wedge 1\right) +\left(1+\frac{3}{H}\right)p_h \pi_{h+1}^{t+1} W_{h+1}^{t}(s,a).
\end{equation*}
We denote by $p_h^{t}(s,a)$ the probability to reach the state-action pair $(s,a)$ at step $h$ under policy $\pi^{t}$ in the true MDP. Unfolding the previous inequality and switching to the pseudo-counts, defined by $\bar n_h^t(s,a) \triangleq \sum_{\ell = 1}^t p_h^{\ell}(s,a)$, by Lemma~\ref{lem:cnt_pseudo} proved in Appendix~\ref{app:count} we get
\begin{equation}
  \label{eq:upper_bound_W1_main}
  \pi_1^{t+1} W_1^t(s_1) \leq 84 e^3 H^2\sum_{h=1}^H \sum_{s,a} p_h^{t+1}(s,a)
  \frac{\beta(\bn_h^t(s,a),\delta)}{\bn_h^t(s,a)\vee 1}\cdot
\end{equation}
Since $t<\tau$ we know that due to stopping rule
\[
 \epsilon \leq 3e \sqrt{\pi_1^{t+1} W_1^{t} (s_1)} + \pi_1^{t+1} W_1^{t}(s_1).
\]
Summing the previous inequalities for $0\leq t< \tau$ then using the Cauchy-Schwarz inequality we obtain
\begin{align*}
\tau \epsilon &\leq \sum_{t=0}^{\tau-1} 3e \sqrt{\pi_1^{t+1} W_1^{t} (s_1)} + \pi_1^{t+1} W_1^{t}(s_1)
\leq 3e \sqrt{\tau \sum_{t=0}^{\tau-1} \pi_1^{t+1} W_1^{t} (s_1) } +  \sum_{t=0}^{\tau-1} \pi_1^{t+1} W_1^{t} (s_1).
\end{align*}
We  now upper-bound the sum that appears in the left-hand terms. Using successively \eqref{eq:upper_bound_W1_main}, $\beta(\cdot,\delta)$ is increasing, Lemma~\ref{lem:sum_1_over_n} of Appendix~\ref{app:count}, we have
\begin{align*}
   \sum_{t=0}^{\tau-1} \pi_1^{t+1} W_1^{t} (s_1) & \leq  84 e^3 H^2\sum_{t=0}^{\tau-1}\sum_{h=1}^H \sum_{s,a} p_h^{t+1}(s,a)
   \frac{\beta(\bn_h^t(s,a),\delta)}{\bn_h^t(s,a)\vee 1}\\
   &\leq 84 e^3 H^2 \beta(\tau-1,\delta)\sum_{h=1}^H \sum_{s,a} \sum_{t=0}^{\tau-1}
   \frac{\bn_h^{t+1}(s,a) - \bn_h^t(s,a)}{\bn_h^t(s,a)\vee 1}\\
   &\leq 336 e^3 H^3 S A \log(\tau+1) \beta(\tau-1,\delta).
\end{align*}
Therefore, combining with the above inequality with the previous one, we get
\[
\tau \epsilon \leq 55e^3 \sqrt{\tau H^3 S A \log(\tau+1) \beta(\tau-1,\delta)} +   336 e^3 H^3 S A \log(\tau+1) \beta(\tau-1,\delta).
\]
Using Lemma \ref{lem:reverse_inequality}, we invert the inequality above and obtain an upper bound on $\tau$, which allows us to conclude the proof of the theorem.

\section{Best-policy identification}
\label{sec:BPI}
Unlike in the previous section, we now consider a more standard setup in which there is a single reward function $r$ and in which the agent observes the reward at each step, during the exploration phase. To ease the presentation, we drop the dependence on the reward $r$ in all data-dependent quantities introduced in this section.

\paragraph{Best-policy identification} In BPI, the agent interacts with the MDP in a way described in Section~\ref{sec:setting}.
Notice that the difference from Section~\ref{sec:reward_free} is that the agent also observes the reward.
In each episode $t$, the agent follows a policy $\pi^t$ (the sampling rule) based only on the information collected up to and including episode $t-1$. At the end of each episode, the agent can decide to stop collecting data (we denote by $\tau$ its random stopping time) and outputs a guess $\hpi$ for the optimal policy.

A BPI algorithm is therefore made of a triple $((\pi^{t})_{t\in\N},\tau,\hpi)$. The goal is to build an $(\varepsilon,\delta)$-PAC algorithm according to the following definition, for which the \emph{sample complexity}, that is the number of exploration episodes $\tau$, is as small as possible.

\begin{definition}[PAC algorithm for BPI] \label{def:BPI} An algorithm is $(\varepsilon,\delta)$-PAC for best policy identification if
it \emph{returns a policy} $\hpi$ after some number of episodes $\tau$ that satisfies
\[\P\left(\Vstar_1(s_1) - V_1^{\hpi}(s_1)\leq \epsilon\right) \geq 1 - \delta.\]
\end{definition}

\subsection{\OurAlgorithmBPI}
\label{ss:bpiucrl}

Similarly to \citet{azar2017minimax} and \citet{zanette2019tighter}, we define upper confidence bounds on the optimal Q-value and value functions as
\begin{align*}
  \utQ_h^{t}(s,a) &\triangleq \min\Bigg(\!H, r_h(s,a) + 3\sqrt{\Var_{\hp_h^{\,t}}(\utV_{h+1}^t)(s,a) \frac{\betastar(n_h^t(s,a),\delta)}{n_h^t(s,a)}}+14H^2 \frac{\beta(n_h^t(s,a),\delta)}{n_h^t(s,a)}\\
  &\qquad+\frac{1}{H}\hp_h^{\,t} (\utV_{h+1}^{t}-\ltV_{h+1}^t)(s,a)   +\hp_h^{\,t} \utV_{h+1}^{t}(s,a)\Bigg),\\
  \utV_h^t(s) &\triangleq \max_{a\in\cA}\utQ_h^t(s,a),\\
  \utV_{H+1}^{t}(s) &\triangleq0,
\end{align*}
where $\betastar$ is some exploration rate (that does \emph{not} have a linear scaling in the number of states $S$, unlike $\beta$) and $\ltV^t$ is a lower confidence bound on the optimal value function; see Appendix~\ref{app:optimism} for a complete definition.

As in RFE, we need to build an upper confidence bound on the gap $\Vstar_1(s_1)-V_1^{\pi^{t+1}}(s_1),$ between the value of the optimal policy and the value of the current policy, to define the stopping rule. We recursively define the functions $G^{t}$ as $G_{H+1}^{\,t}(s,a) \triangleq 0$ for all $(s,a)$ and for all $(s,a,h)$ as
\begin{align*}
  G_h^{\,t}(s,a) &\triangleq \min\!\!\Bigg(\!H,  6\sqrt{\Var_{\hp_h^t}(\utV_{h+1}^t)(s,a) \frac{\betastar(n_h^t(s,a),\delta)}{n_h^t(s,a)}}+ 36 H^2 \frac{\beta(n_h^t(s,a),\delta)}{n_h^t(s,a)}\\
  &\qquad+ \left(1+\frac{3}{H}\right)\hp_h^{\,t} \pi_{h+1}^{t+1} G_{h+1}^t(s,a)\!\Bigg)\quad \text{for $h \leq H$}.
\end{align*}
We prove the following result in Appendix~\ref{app:proof_BPI}.
\begin{restatable}{lemma}{lemmaBPI}
 \label{lem:control_gap_BPI}
With probability at least $1-\delta$, for all $t$,
\[
\Vstar_1(s_1) - V_1^{\pi^{t+1}}(s_1) \leq \pi_1^{t+1}G_1^t(s_1)\,.
\]
In particular it holds on the event $\cG$ defined in Appendix~\ref{app:concentration}.
\end{restatable}
\noindent
We are now ready to define our \OurAlgorithmBPI algorithm.
\begin{framed}
\begin{itemize}[leftmargin=9em]
\item[\bf sampling rule:] the policy $\pi^{t+1}$ is the greedy policy with respect to $\utQ_h^{t}$,
  \[\forall s \in \cS, \forall h \in [H], \ \ \pi^{t+1}_h(s) = \argmax_{a\in\cA} \utQ_h^{t}(s,a)\]
    \vspace{-2em}
 \item[\bf  stopping rule:]  $\tau =\inf \left\{ t\in \N : \pi_1^{t+1}G_1^t(s_1)  \leq \epsilon \right\}$
 \item[\bf  prediction rule:]  $\hpi = \pi^{\tau+1}$
\end{itemize}
\end{framed}
\noindent
We provide a sample complexity bound for \OurAlgorithmBPI in the next theorem, which we prove in Appendix~\ref{app:proof_BPI}.
\begin{restatable}{theorem}{theoremBPI}\label{thm:sample_complexity_BPI} For $\delta\in(0,1)$, $\epsilon\in (0,1/S^2]$, \OurAlgorithmBPI using thresholds $\beta(n,\delta) \triangleq \log\!\big(3SAH/\delta\big) + S\log\big(8e(n+1)\big)$ and $\betastar(n,\delta) \triangleq \log(3SAH/\delta) + \log\!\big(8e(n+1)\big)$
is $(\epsilon,\delta)$-PAC for best policy exploration. Moreover, with probability $1-\delta$,
\begin{align*}
\tau \leq \frac{H^3SA}{\epsilon^2} \big( \log(3SAH/\delta) +1\big) C_1+1,
\end{align*}
where $C_1 \triangleq 5904e^{26} \log\left(e^{30} \left( \log(3SAH/\delta) +S\right) H^3SA/\epsilon\right)^2.$
\end{restatable}

\noindent Therefore, the rate of \OurAlgorithmBPI is of order $\tcO\big(H^3SA\log(1/\delta)/\epsilon^2\big)$ when $\epsilon$ is small
enough and matches the lower bound of $\Omega\big(H^2SA\log(1/\delta)/\epsilon^2\big)$ by \citet{dann2015sample} up to an $H$ and poly-log terms. To the best of our knowledge, \OurAlgorithmBPI is the first algorithm for BPI whose sample complexity has an optimal dependence on $S,A, \varepsilon,$ and $\delta$. In Remark \ref{remark:optimality}, we discuss the optimality of \OurAlgorithmBPI with respect to $H$.

\paragraph{From regret-minimization to BPI} The main difficulty for converting a regret-minimizing algorithm to BPI lies in high-probability prediction of an $\epsilon$-optimal policy. Indeed, assume that a regret-minimizing algorithm generates a sequence of policy $(\pi^t)_{t\in[T]}$ for a number of episodes $T\in\N$ large enough, with a controlled regret such that with probability at least $1-\delta',$
\[
\sum_{t=1}^T \Vstar(s_1)-V_1^{\pi^t}(s_1) \leq C\sqrt{H^3SA \log(1/\delta') T},
\]
for some constant $C$. Then a straightforward choice for the prediction is to choose $\hpi$ at random among  the sequence $(\pi^t)_{t\in[T]}$, as suggested by \cite{jin2018is}. Markov's inequality implies that
\[
\P\left( \Vstar_1(s_1)-V_1^{\pi^t}(s_1) \geq \epsilon \right) \leq \frac{1}{\epsilon} \E\left[\frac{1}{T}\sum_{t=1}^T \Vstar(s_1)-V_1^{\pi^t}(s_1)\right]\leq \frac{1}{\epsilon}\left(C \sqrt{\frac{H^3 S A}{T} \log(1/\delta')  }+\delta'H \right).
\]
Therefore, if we choose $\delta' \triangleq \frac{\epsilon}{H}\cdot\frac{\delta}{2}$ and
\[T \triangleq \frac{2H^3SA}{\epsilon^2 \delta^2} \log\left(\frac{2H}{\epsilon \delta}\right)\CommaBin\]
we get $\P\big( \Vstar(s_1)-V_1^{\pi^t}(s_1) \geq \epsilon \big) \leq \delta$ which implies that the described algorithm is $(\epsilon,\delta)$-PAC for BPI. Unfortunately, the choice of $T$ above gives a sample complexity that scales with $1/\delta^2\!\,$ whereas we prefer and expect $\log(1/\delta)$.

 Another natural choice for the prediction is to return the average policy.
Precisely, let us define $\bp_h^{\,t}(s,a) \triangleq \sum_{\ell=1}^t p_h^\ell(s,a)$ be the average state-action distribution, we define the average distribution as the one for which its state-action distribution is precisely $\bp^t$,
\[
\bpi_h^t(a|s) \triangleq \begin{cases}
  \frac{\bp_h^t(s,a)}{\sum_{b\in\cA}\bp_h^t(s,b)}&\text{ if }\sum_{b\in\cA}\bp_h^t(s,b) >0, \  \text{and} \\
  1/A&\text{\ otherwise.}
\end{cases}
\]
Note that the average policy $\bpi_h^t(a|s)$ is not the average of the policies $\sum_{\ell=1}^t \pi_h^\ell(a|s)/t$. The interest of $\bpi_h^t$ is that the average regret at time $T$ is \emph{equal} to the simple regret of the policy $\bpi^T$,
\[
 \Vstar_1(s_1)-V_1^{\bpi^T}(s_1) \leq \frac{1}{T}\sum_{t=1}^T \Vstar(s_1)-V_1^{\pi^t}(s_1) \leq C\sqrt{\frac{H^3SA}{T} \log(1/\delta')},
\]
with probability $1-\delta'$. Choosing $\delta'\triangleq\delta$ and $T \triangleq CH^3SA\log(1/\delta)/\epsilon^2$ and predicting policy~$\bpi^T$ would lead to an $(\epsilon,\delta)$-PAC algorithm for BPI with a minimax optimal sample complexity. However, since the agent does not know the transition kernel, it cannot compute the average policy~$\bpi^T$! Therefore, \OurAlgorithmBPI rather relies on an upper bound on the simple regret (see Lemma~\ref{lem:control_gap_BPI}) computable by the agent to detect if the policy currently used by the sampling rule is $\epsilon$-optimal.

\paragraph{Acknowledgments}
The research presented was supported by European CHIST-ERA project DELTA, French Ministry of
Higher Education and Research, Nord-Pas-de-Calais Regional Council,  French National Research Agency project BOLD (ANR19-CE23-0026-04).
Anders Jonsson is partially supported by the Spanish grants TIN2015-67959 and PCIN-2017-082.
\bibliographystyle{plainnat}
\bibliography{ActiveExplorationMDP_bib}

\begin{thebibliography}{30}
\providecommand{\natexlab}[1]{#1}
\providecommand{\url}[1]{\texttt{#1}}
\expandafter\ifx\csname urlstyle\endcsname\relax
  \providecommand{\doi}[1]{doi: #1}\else
  \providecommand{\doi}{doi: \begingroup \urlstyle{rm}\Url}\fi

\bibitem[Agarwal et~al.(2020)Agarwal, Kakade, and Yang]{agarwal2020model}
Alekh Agarwal, Sham Kakade, and Lin~F Yang.
\newblock \href{https://arxiv.org/pdf/1906.03804.pdf}{{Model-based
  reinforcement learning with a generative model is minimax optimal}}.
\newblock In \emph{Conference on Learning Theory}, 2020.

\bibitem[Azar et~al.(2012)Azar, Munos, and Kappen]{azar2012on}
Mohammad~Gheshlaghi Azar, R{\'{e}}mi Munos, and Bert Kappen.
\newblock \href{https://arxiv.org/pdf/1206.6461.pdf}{{On the sample complexity
  of reinforcement learning with a generative model}}.
\newblock In \emph{International Conference on Machine Learning}, 2012.

\bibitem[Azar et~al.(2013)Azar, Munos, and Kappen]{azar2013minimax}
Mohammad~Gheshlaghi Azar, R{\'{e}}mi Munos, and Hilbert~J. Kappen.
\newblock \href{https://hal.archives-ouvertes.fr/hal-00831875}{{Minimax PAC
  bounds on the sample complexity of reinforcement learning with a generative
  model}}.
\newblock \emph{Machine Learning}, 91\penalty0 (3):\penalty0 325--349, 2013.

\bibitem[Azar et~al.(2017)Azar, Osband, and Munos]{azar2017minimax}
Mohammad~Gheshlaghi Azar, Ian Osband, and R{\'{e}}mi Munos.
\newblock \href{https://arxiv.org/pdf/1703.05449.pdf}{{Minimax regret bounds
  for reinforcement learning}}.
\newblock In \emph{International Conference on Machine Learning}, 2017.

\bibitem[Bartlett and Tewari(2009)]{bartlett2009regal}
Peter~L. Bartlett and Ambuj Tewari.
\newblock \href{https://arxiv.org/pdf/1205.2661.pdf}{{REGAL: A regularization
  based algorithm for reinforcement learning in weakly communicating MDPs}}.
\newblock In \emph{Uncertainty in Artificial Intelligence}, 2009.

\bibitem[Boucheron et~al.(2013)Boucheron, Lugosi, and
  Massart]{boucheron2013concentration}
St{\'{e}}phane Boucheron, G{\'{a}}bor Lugosi, and Pascal Massart.
\newblock \href{https://www.hse.ru/data/2016/11/24/1113029206/Concentration
  inequalities.pdf}{\emph{{Concentration inequalities}}}.
\newblock Oxford University Press, 2013.

\bibitem[Cover and Thomas(2006)]{cover2006elements}
Thomas~M. Cover and Joy~A. Thomas.
\newblock
  \href{https://www.amazon.com/Elements-Information-Theory-Telecommunications-Processing/dp/0471241954}{\emph{{Elements
  of information theory}}}.
\newblock John Wiley {\&} Sons, 2006.

\bibitem[Dann and Brunskill(2015)]{dann2015sample}
Christoph Dann and Emma Brunskill.
\newblock \href{https://arxiv.org/pdf/1510.08906.pdf}{{Sample complexity of
  episodic fixed-horizon reinforcement learning}}.
\newblock In \emph{Neural Information Processing Systems}, 2015.

\bibitem[Dann et~al.(2017)Dann, Lattimore, and Brunskill]{dann2017unifying}
Christoph Dann, Tor Lattimore, and Emma Brunskill.
\newblock \href{https://arxiv.org/pdf/1703.07710.pdf}{{Unifying {PAC} and
  regret: Uniform {PAC} bounds for episodic reinforcement learning}}.
\newblock In \emph{Neural Information Processing Systems}, 2017.

\bibitem[Dann et~al.(2019)Dann, Li, Wei, and Brunskill]{dann2019policy}
Christoph Dann, Lihong Li, Wei Wei, and Emma Brunskill.
\newblock \href{https://arxiv.org/pdf/1811.03056.pdf}{{Policy certificates:
  Towards accountable reinforcement learning}}.
\newblock In \emph{International Conference on Machine Learning}, 2019.

\bibitem[de~la Pe{\~{n}}a et~al.(2004)de~la Pe{\~{n}}a, Klass, and
  Lai]{de2004self}
Victor~H. de~la Pe{\~{n}}a, Michael~J. Klass, and Tze~Leung Lai.
\newblock \href{https://arxiv.org/pdf/math/0410102.pdf}{{Self-normalized
  processes: {E}xponential inequalities, moment bounds and iterated logarithm
  laws}}.
\newblock \emph{Annals of probability}, 32:\penalty0 1902--1933, 2004.

\bibitem[Domingues et~al.(2020)Domingues, M{\'e}nard, Pirotta, Kaufmann, and
  Valko]{domingues2020regret}
Omar~Darwiche Domingues, Pierre M{\'e}nard, Matteo Pirotta, Emilie Kaufmann,
  and Michal Valko.
\newblock \href{http://arxiv.org/abs/2004.05599}{Regret bounds for kernel-based
  reinforcement learning}.
\newblock \emph{arXiv preprint arXiv:2004.05599}, 2020.

\bibitem[Durrett(2010)]{durrett2010probability}
Rick Durrett.
\newblock
  \href{https://services.math.duke.edu/{~}rtd/PTE/PTE5{\_}011119.pdf}{\emph{{Probability:
  Theory and Examples}}}.
\newblock Cambridge Series in Statistical and Probabilistic Mathematics.
  Cambridge University Press, 4 edition, 2010.

\bibitem[Even-Dar et~al.(2006)Even-Dar, Mannor, and Mansour]{even2006action}
Eyal Even-Dar, Shie Mannor, and Yishay Mansour.
\newblock
  \href{https://jmlr.csail.mit.edu/papers/volume7/evendar06a/evendar06a.pdf}{{Action
  elimination and stopping conditions for the multi-armed bandit and
  reinforcement learning problems}}.
\newblock \emph{Journal of Machine Learning Research}, 7:\penalty0 1079--1105,
  2006.

\bibitem[Fiechter(1994)]{fiechter1994efficient}
Claude-Nicolas Fiechter.
\newblock
  \href{http://citeseerx.ist.psu.edu/viewdoc/download;jsessionid=7F5F8FCD1AA7ED07356410DDD5B384FE?doi=10.1.1.49.8652\&rep=rep1\&type=pdf}{{Efficient
  reinforcement learning}}.
\newblock In \emph{Conference on Learning Theory}, 1994.

\bibitem[Garivier et~al.(2019)Garivier, M{\'{e}}nard, and
  Stoltz]{garivier2019explore}
Aur{\'{e}}lien Garivier, Pierre M{\'{e}}nard, and Gilles Stoltz.
\newblock \href{https://arxiv.org/pdf/1602.07182.pdf}{{Explore first, exploit
  next: The true shape of regret in bandit problems}}.
\newblock \emph{Mathematics of Operations Research}, 44\penalty0 (2):\penalty0
  377--399, 2019.

\bibitem[Hazan et~al.(2019)Hazan, Kakade, Singh, and Soest]{hazan2019provably}
Elad Hazan, Sham Kakade, Karan Singh, and Abby~Van Soest.
\newblock \href{https://arxiv.org/pdf/1812.02690.pdf}{{Provably efficient
  maximum entropy exploration}}.
\newblock In \emph{International Conference on Machine Learning}, 2019.

\bibitem[Jaksch et~al.(2010)Jaksch, Ortner, and Auer]{jaksch2010near}
Thomas Jaksch, Ronald Ortner, and Peter Auer.
\newblock
  \href{http://www.jmlr.org/papers/volume11/jaksch10a/jaksch10a.pdf}{{Near-optimal
  regret bounds for reinforcement learning}}.
\newblock \emph{Journal of Machine Learning Research}, 99:\penalty0 1563--1600,
  2010.

\bibitem[Jin et~al.(2018)Jin, Allen-Zhu, Bubeck, and Jordan]{jin2018is}
Chi Jin, Zeyuan Allen-Zhu, S{\'{e}}bastien Bubeck, and Michael~I. Jordan.
\newblock \href{https://arxiv.org/pdf/1807.03765.pdf}{{Is Q-learning provably
  efficient?}}
\newblock In \emph{Neural Information Processing Systems}, 2018.

\bibitem[Jin et~al.(2020)Jin, Krishnamurthy, Simchowitz, and
  Yu]{jin2020reward-free}
Chi Jin, Akshay Krishnamurthy, Max Simchowitz, and Tiancheng Yu.
\newblock \href{http://arxiv.org/abs/2002.02794}{{Reward-free exploration for
  reinforcement learning}}.
\newblock In \emph{International Conference on Machine Learning}, 2020.

\bibitem[Jonsson et~al.(2020)Jonsson, Kaufmann, M{\'{e}}nard, Domingues,
  Leurent, and Valko]{jonsson2020planning}
Anders Jonsson, Emilie Kaufmann, Pierre M{\'{e}}nard, Omar~Darwiche Domingues,
  Edouard Leurent, and Michal Valko.
\newblock \href{http://arxiv.org/abs/2006.05879}{{Planning in markov decision
  processes with gap-dependent sample complexity}}.
\newblock \emph{arXiv preprint arXiv:2006.05879},  2020.

\bibitem[Kakade(2003)]{kakade2013on}
Sham Kakade.
\newblock
  \href{https://homes.cs.washington.edu/~sham/papers/thesis/sham_thesis.pdf}{\emph{{On
  the sample complexity of reinforcement learning}}}.
\newblock PhD thesis, University College London, 2003.

\bibitem[Kaufmann et~al.(2020)Kaufmann, M{\'{e}}nard, Domingues, Jonsson,
  Leurent, and Valko]{kaufmann2020adaptive}
Emilie Kaufmann, Pierre M{\'{e}}nard, Omar~Darwiche Domingues, Anders Jonsson,
  Edouard Leurent, and Michal Valko.
\newblock \href{https://arxiv.org/pdf/2006.06294.pdf}{{Adaptive reward-free
  exploration}}.
\newblock \emph{arXiv preprint arXiv:2006.06294}, 2020.

\bibitem[Kearns and Singh(1998)]{kearns1998finite-sample}
Michael~J. Kearns and Satinder~P. Singh.
\newblock
  \href{http://papers.neurips.cc/paper/1531-finite-sample-convergence-rates-for-q-learning-and-indirect-algorithms.pdf}{{Finite-sample
  convergence rates for Q-learning and indirect algorithms}}.
\newblock In \emph{Neural Information Processing Systems}, 1998.

\bibitem[Puterman(1994)]{puterman1994markov}
Martin~L. Puterman.
\newblock
  \href{https://onlinelibrary.wiley.com/doi/book/10.1002/9780470316887}{\emph{{Markov
  decision processes: Discrete stochastic dynamic programming}}}.
\newblock John Wiley {\&} Sons, New York, NY, 1994.

\bibitem[Sidford et~al.(2018)Sidford, Wang, Wu, Yang, and Ye]{sidford2018near}
Aaron Sidford, Mengdi Wang, Xian Wu, Lin~F. Yang, and Yinyu Ye.
\newblock \href{https://arxiv.org/pdf/1806.01492.pdf}{{Near-optimal time and
  sample complexities for solving discounted Markov decision process with a
  generative model}}.
\newblock In \emph{Neural Information Processing Systems}, 2018.

\bibitem[Talebi and Maillard(2018)]{talebi2018variance}
Mohammad~Sadegh Talebi and Odalric-Ambrym Maillard.
\newblock \href{https://arxiv.org/pdf/1803.01626.pdf}{{Variance-aware regret
  bounds for undiscounted reinforcement learning in MDPs}}.
\newblock In \emph{Algorithmic Learning Theory}, 2018.

\bibitem[Wang et~al.(2020)Wang, Du, Yang, and Salakhutdinov]{wang2020on}
Ruosong Wang, Simon~S Du, Lin~F Yang, and Ruslan Salakhutdinov.
\newblock \href{https://arxiv.org/pdf/2006.11274.pdf}{{On reward-free
  reinforcement learning with linear function approximation}}.
\newblock \emph{arXiv preprint arXiv:2006.11274}, 2020.

\bibitem[Zanette and Brunskill(2019)]{zanette2019tighter}
Andrea Zanette and Emma Brunskill.
\newblock \href{https://arxiv.org/pdf/1901.00210.pdf}{{Tighter
  problem-dependent regret bounds in reinforcement learning without domain
  knowledge using value function bounds}}.
\newblock In \emph{International Conference on Machine Learning}, 2019.

\bibitem[Zhang et~al.(2020)Zhang, Ma, and Singla]{zhang2020task-agnostic}
Xuezhou Zhang, Yuzhe Ma, and Adish Singla.
\newblock \href{https://arxiv.org/pdf/2006.09497.pdf}{{Task-agnostic
  exploration in reinforcement learning}}.
\newblock \emph{arXiv preprint: arXiv:2006.09497}, 2020.

\end{thebibliography}

\newpage
\appendix

\part{Appendix}
\parttoc

\newpage
\section{Concentration events}
\label{app:concentration}
We define the following favorable events: $\cE$ the event where the empirical transition probabilities are close to the true ones, $\cE^\cnt$ the event where the pseudo-counts are close to their expectation, and $\cE^\star$ where the empirical means of the optimal value functions are close to the true ones,
\begin{align*}
 \cE &\triangleq \left\{\forall t \in \N, \forall h \in [H], \forall (s,a)\in\cS\times\cA:\ \KL\left(\widehat{p}^{\,t}_h(\cdot | s,a), p_h(\cdot | s,a)\right)\leq \frac{\beta(n_h^t(s,a),\delta)}{n_h^t(s,a)}\right\}\CommaBin \\
 \cE^{\cnt} &\triangleq  \left\{ \forall t \in \N, \forall h\in [H],\forall (s ,a)\in\cS\times\cA:\ n_h^t(s,a) \geq \frac{1}{2}\bar n_h^t(s,a)-\beta^{\cnt}(\delta)  \right\}\!\CommaBin \ \text{and}\\
 \cE^\star &\triangleq \Bigg\{\forall t \in \N, \forall h \in [H], \forall (s,a)\in\cS\times\cA:\\
  &  \big|(\hp_h^{\,t}-p_h)\Vstar_{h+1}(s,a)\big|\leq\min\left(\!H, \sqrt{2\Var_{p_h}(\Vstar_{h+1})(s,a)\frac{\betastar(n_h^t(s,a),\delta) }{n_h^t(s,a)}}+3 H \frac{\betastar(n_h^t(s,a),\delta) }{n_h^t(s,a)}\right)\Bigg\}\cdot
\end{align*}
We also introduce the intersection of these events, $\cG \triangleq \cE \cap \cE^{\cnt}\cap \cE^\star,$ and the intersection of only the first two events, $\cF \triangleq \cE \cap \cE^{\cnt}$ . Note that the event $\cF$ is independent of the reward function $r$. We  prove that for the right choice of the functions $\beta$ the above events hold with high probability.
\begin{lemma}
\label{lem:proba_master_event}
For the following choices of functions $\beta,$
\begin{align*}
  \beta(n,\delta) &\triangleq   \log(3SAH/\delta) + S\log \left(8e(n+1)\right),\\
  \beta^\cnt(\delta) &\triangleq \log\left(3SAH/\delta\right), \quad \text{and}\\
  \betastar(n,\delta) &\triangleq \log(3SAH/\delta) + \log\left(8e(n+1)\right)\,,\\
\end{align*}
it holds that
\[
\P(\cE)\geq 1-\delta, \qquad \P(\cE^{\cnt})\geq 1-\delta,  \qquad \text{and} \qquad \P(\cE^\star)\geq 1-\delta\,.
\]
In particular, $\P(\cG) \geq 1-\delta$ and $\P(\cF) \geq 1-\delta$.
\end{lemma}
\begin{proof}
First, by Theorem~\ref{th:max_ineq_categorical}, we have that
\[
\P(\cE)\geq 1-\frac{\delta}{3}\cdot
\]
Second, by Theorem~\ref{th:bernoulli-deviation}, we have that
\[
\P(\cE^{\cnt})\geq 1-\frac{\delta}{3}\cdot
\]
Finally, by Theorem~\ref{th:bernstein}, we have that
\[
\P(\cE^\star)\geq 1-\frac{\delta}{3}\cdot
\]
Applying a union to the above three inequalities, we conclude that
\[
\P(\cG)\geq 1-\delta \qquad \P(\cE)\geq 1-\delta\,.
\]
\noindent \emph{Remark}: Note that we can order 
  $1\leq \beta^{\cnt}(\delta)\leq \betastar(n,\delta) \leq \beta(n,\delta).$
\end{proof}

\subsection{Deviation inequality for categorical distributions}
Next, we reproduce the deviation inequality for categorical distributions by \citet[Proposition 1]{jonsson2020planning}.
 Let $(X_t)_{t\in\N^\star}$ be i.i.d.\,samples from a distribution supported on $\{1,\ldots,m\}$, of probabilities given by $p\in\Sigma_m$, where $\Sigma_m$ is the probability simplex of dimension $m-1$. We denote by $\hp_n$ the empirical vector of probabilities, i.e., for all $k\in\{1,\ldots,m\},$
 \[
 \hp_{n,k} = \frac{1}{n} \sum_{\ell=1}^n \ind\left\{X_\ell = k\right\}.
 \]
 Note that  an element $p \in \Sigma_m$ can be seen as an element of $\R^{m-1}$ since $p_m = 1- \sum_{k=1}^{m-1} p_k$. This will be clear from the context. We denote by $H(p)$ the (Shannon) entropy of $p\in\Sigma_m$,
 \[
 H(p) = \sum_{k=1}^m p_k \log\left(\frac{1}{p_k}\right)\cdot
 \]
 \begin{theorem} \label{th:max_ineq_categorical}
 For all $p\in\Sigma_m$ and for all $\delta\in[0,1]$,
 \begin{align*}
     \P\left(\exists n\in \N^\star,\, n\KL(\hp_n, p)> \log(1/\delta) + (m-1)\log\left(e(1+n/(m-1))\right)\right)\leq \delta.
 \end{align*}
\end{theorem}

 \begin{proof}
 We apply the method of mixtures with a Dirichlet prior on the mean parameter of the exponential family formed by the set of categorical distribution on  $\{1,\ldots,m\}$. Letting
 \[\phi_p(\lambda) \triangleq \log \E_{X\sim p}\left[e^{\lambda X}\right] =\log\left(\!p_m+\sum_{k=1}^{m-1}p_k e^{\lambda_k}\!\right)\]
 be the log-partition function, we have that  the following $M_n^\lambda$  is a martingale,
 \[
 M_n^\lambda \triangleq e^{n \langle \lambda,\hp_n\rangle - n \phi_p(\lambda)}.
 \]
 We set a Dirichlet prior $q\sim \Dir(\alpha)$ with $\alpha \in {\R^\star_+}^m$ and for $\lambda_q \triangleq (\nabla\phi_p)^{-1}(q),$  consider the integrated martingale
 \begin{align*}
     M_n &= \int M_n^{\lambda_q} \frac{\Gamma\left(\sum_{k=1}^m \alpha_k\right)}{\prod_{k=1}^m\Gamma(\alpha_k)} q_k^{\alpha_k-1}\diff q\\
     &= \int e^{n\left(\KL(\hp_n,p)-\KL(\hp_n,q)\right)}\frac{\Gamma\left(\sum_{k=1}^m \alpha_k\right)}{\prod_{k=1}^m\Gamma(\alpha_k)} q_k^{\alpha_k-1}\diff q\\
     &=  e^{n\KL(\hp_n,p) + n H(\hp_n)} \int \frac{\Gamma\left(\sum_{k=1}^m \alpha_k\right)}{\prod_{k=1}^m\Gamma(\alpha_k)} q_k^{n\hp_{n,k}+\alpha_k-1}\diff q\\
     &= e^{n\KL(\hp_n,p) + n H(\hp_n)}\frac{\Gamma\left(\sum_{k=1}^m \alpha_k\right)}{\prod_{k=1}^m\Gamma(\alpha_k)} \frac{\prod_{k=1}^m\Gamma(\alpha_k + n \hp_{n,k})}{\Gamma\left(\sum_{k=1}^m \alpha_k + n\right)}\CommaBin
 \end{align*}
 where in the second inequality we used Lemma~\ref{lem:martingale_to_diff_kl}. Next, we choose the uniform prior $\alpha= (1,\ldots,1)$ to obtain
 \begin{align*}
 M_n &= e^{n\KL(\hp_n,p) + n H(\hp_n)} (m-1)! \frac{\prod_{k=1}^m\Gamma(1 + n \hp_{n,k})}{\Gamma( m + n)}\\
 &= e^{n\KL(\hp_n,p) + n H(\hp_n)} (m-1)! \frac{\prod_{k=1}^m(n \hp_{n,k})!}{n!} \frac{n!}{(m+n-1)!}\\
 & = e^{n\KL(\hp_n,p) + n H(\hp_n)} \frac{1}{\binom{n}{n \hp_n}} \frac{1}{\binom{m+n-1}{m-1}}\cdot
 \end{align*}
 Theorem~11.1.3 by \citet{cover2006elements} shows us how to upper-bound the multinomial coefficient. In particular, for $M\in \N^\star$ and $x\in \{0,\ldots,M\}^m$ such that $\sum_{k=1}^m x_k = M,$
 \[
 \binom{M}{x} = \frac{M!}{\prod_{k=1}^m x_k!} \leq e^{M H(x/M)}.
 \]
 Using this inequality we obtain
 \begin{align*}
 M_n &\geq e^{n\KL(\hp_n,p) +n H(\hp_n) - n H(\hp_n) -(m+n-1) H\left((m-1)/(m+n-1)\right)} \\
 &=  e^{n\KL(\hp_n,p)-(m+n-1) H\left((m-1)/(m+n-1)\right)}.
 \end{align*}
 It remains to upper-bound entropic term which we do as follows,
 \begin{align*}
    (m+n-1) H\left((m-1)/(m+n-1)\right) &= (m-1) \log \frac{m+n-1}{m-1} + n \log \frac{m+n-1}{n}\\
    &\leq  (m-1) \log\left( 1 + n/(m-1)\right) + n \log(1+(m-1)/n)\\
    &\leq (m-1) \log\left( 1 + n/(m-1)\right) + (m-1).
 \end{align*}
 Therefore, we lower bound the martingale as
 \begin{align*}
     M_n &\geq e^{n\KL(\hp_n,p)}\left( e(1+n/(m-1))  \right)^{m-1}.
 \end{align*}
Since for any supermartingale we have that
 \begin{equation}\P\left(\exists n \in \N^\star : M_n > 1/\delta\right) \leq \delta \cdot \E[M_1],\label{eq:supermartingale}\end{equation}
 which is a well-known property of the method of mixtures \citep{de2004self}, we conclude that
 \[
 \P\left(\exists n\in \N^\star,\, n\KL(\hp_n, p)> (m-1)\log\left(e(1+n/(m-1))\right)+ \log(1/\delta)\right) \leq \delta\,.
 \]
 \end{proof}
 \begin{lemma}
 \label{lem:martingale_to_diff_kl}
 For $q,p \in \Sigma_m$ and $\lambda\in \R^{m-1}$,
 \[
 \langle \lambda,q \rangle -\phi_{p}(\lambda) = \KL(q,p) -  \KL(q,p^\lambda)\,,
 \]
 where $\phi_p(\lambda) = \log(p_m+\sum_{k=1}^{m-1}p_k e^{\lambda_k})$ and $p^\lambda = \nabla \phi_{p_0}(\lambda)$.
 \end{lemma}
 \begin{proof}
 First note that
  \[
  p^\lambda_k = \frac{p_k e^{\lambda_k}}{p_m +  \sum_{\ell=1}^{m-1} p_\ell e^{\lambda_\ell}}\CommaBin
  \]
  which implies that
  \[
  p_m + \sum_{k=1}^{m-1} p_k e^{\lambda_k} = \frac{p_m}{p_m^\lambda}\CommaBin \qquad \lambda_k = \log\frac{p_k^\lambda}{p_k} + \log \frac{p_m}{p_m^\lambda}\cdot
  \]
  Therefore, we obtain
  \begin{align*}
  \langle \lambda, q\rangle -\phi_p (\lambda) &= \sum_{k=1}^{m-1} q_k \log \left(\frac{p_k^\lambda}{p_k} \frac{p_m}{p_m^\lambda}\right) - \log\left(p_m + \sum_{k=1}^{m-1} p_k e^{\lambda_k}\right)\\
  &= \sum_{k=1}^{m-1} q_k \log\frac{p_k^\lambda}{p_k} +(1-q_m) \log \frac{p_m}{p_m^\lambda} - \log\frac{p_m}{p_m^\lambda}\\
  &= \sum_{k=1}^m q_k \log\frac{p_k^\lambda}{p_k} = \KL(q,p)-\KL(q,p^\lambda)\,.
  \end{align*}
 \end{proof}

\subsection{Deviation inequality for sequence of Bernoulli random variables}

Below, we reproduce the deviation inequality for Bernoulli distributions by \citet[Lemma F.4]{dann2017unifying}.
Let $\mathcal F_t$ for $t\in\N$ be a filtration and $(X_t)_{t\in\N^\star}$ be a sequence of Bernoulli random variables with $\P(X_t = 1 | \mathcal F_{t-1}) = P_t$ with $P_t$ being $\mathcal F_{t-1}$-measurable and $X_t$ being $\mathcal F_{t}$-measurable.

\begin{theorem}
	For all $\delta>0$,
	\begin{align}
	\P \left(\exists n : \,\, \sum_{t=1}^n X_t < \sum_{t=1}^n P_t / 2 -\log\frac{1}{\delta}  \right) \leq \delta.
	\end{align}
	\label{th:bernoulli-deviation}
\end{theorem}

\begin{proof}
	$P_t - X_t$ is a martingale difference sequence with respect to the filtration $\mathcal F_t$.
	Since $X_t$ is nonnegative and has finite second moment, we have for any $\lambda > 0$ that
	$\E\left[e^{-\lambda (X_t - P_t)} | \mathcal F_{t-1} \right] \leq e^{\lambda^2 P_t / 2}$ (Exercise 2.9, \citealp{boucheron2013concentration}).
	Hence, we have
	\begin{align}
	\E\left[ e^{\lambda(P_t - X_t) - \lambda^2 P_t / 2} | \mathcal F_{t-1} \right] \leq 1
	\end{align}
	and by setting $\lambda \triangleq 1$, we see that
	\begin{align}
	M_n = e^{\sum_{t=1}^n (-X_t + P_t / 2)}
	\end{align}
	is a supermartingale.
	Therefore by Markov's inequality,
	\begin{align}
	\P \left( \sum_{t=1}^n (-X_t + P_t / 2) \geq \log\frac{1}{\delta} \right)
	=
	\P \left( M_n \geq \frac{1}{\delta} \right)
	\leq {\delta} \E[M_n] \leq \delta
	\end{align}
	which gives us
	\begin{align}
	\P \left( \sum_{t=1}^n X_t \leq \sum_{t=1}^n P_t / 2 -\log\frac{1}{\delta}\right)
	\leq \delta
	\end{align}
	for a fixed $n$.
	We define now the stopping time $\tau \triangleq \min\{ t \in \mathbb N \, : \, M_t > \frac{1}{\delta} \}$ and the sequence $\tau_n = \min\{ t \in \mathbb N \, : \, M_t > \frac{1}{\delta} \vee t \geq n \}$.
	Applying the convergence theorem for nonnegative supermartingales
	(Theorem~5.2.9 by \citealp{durrett2010probability}), we get that $\lim_{t \rightarrow
		\infty} M_t$ is well-defined almost surely. Therefore, $M_\tau$ is
	well-defined even when $\tau = \infty$.
	By the optional stopping theorem for nonnegative supermartingales (Theorem
	5.7.6 by \citealp{durrett2010probability}), we have $\E[M_{\tau_n}] \leq \E[M_0] \leq
	1$ for all $n$ and applying Fatou's lemma, we obtain
	$\E[M_\tau] = \E[\lim_{n \rightarrow \infty} M_{\tau_n}] \leq \lim \inf_{n \rightarrow \infty} \E[M_{\tau_n}] \leq 1$.
	Using Markov's inequality, we can finally bound
	\begin{align}
	\P\left( \exists n:\,\, \sum_{t=1}^n X_t < \frac 1 2 \sum_{t=1}^n P_t - \log\frac{1}{\delta} \right)
	\leq
	\P ( \tau < \infty) \leq \P ( M_{\tau} > \frac{1}{\delta})
	\leq \delta \E[M_{\tau}] \leq \delta.
	\end{align}
\end{proof}

\subsection{Deviation inequality for bounded distributions}
Below, we reproduce the self-normalized Bernstein-type inequality by \citet{domingues2020regret}. Let $(Y_t)_{t\in\N^\star}$, $(w_t)_{t\in\N^\star}$ be two sequences of random variables adapted to a filtration $(\cF_t)_{t\in\N}$. We assume that the weights are in the unit interval $w_t\in[0,1]$ and predictable, i.e. $\cF_{t-1}$ measurable. We also assume that the random variables $Y_t$  are bounded $|Y_t|\leq b$ and centered $\EEc{Y_t}{\cF_{t-1}} = 0$.
Consider the following quantities
\begin{align*}
		S_t \triangleq \sum_{s=1}^t w_s Y_s, \quad V_t \triangleq \sum_{s=1}^t w_s^2\cdot\EEc{Y_s^2}{\cF_{s-1}}, \quad \mbox{and} \quad W_t \triangleq \sum_{s=1}^t w_s
\end{align*}
and let $h(x) \triangleq (x+1) \log(x+1)-x$ be the Cramér transform of a Poisson distribution of parameter~1.

\begin{theorem}[Bernstein-type concentration inequality]
  \label{th:bernstein}
	For all $\delta >0$,
	\begin{align*}
		\PP{\exists t\geq 1,   (V_t/b^2+1)h\left(\!\frac{b |S_t|}{V_t+b^2}\right) \geq \log(1/\delta) + \log\left(4e(2t+1)\!\right)}\leq \delta.
	\end{align*}
  The previous inequality can be weakened to obtain a more explicit bound: with probability at least $1-\delta$, for all $t\geq 1$,
 \[
 |S_t|\leq \sqrt{2V_t \log\left(4e(2t+1)/\delta\right)}+ 3b\log\left(4e(2t+1)/\delta\right)\,.
 \]
\end{theorem}

\begin{proof}
	By homogeneity we can assume that $b=1$ to prove the first part. First note that for all $\lambda > 0$,
	\[
	 e^{\lambda w_t Y_t} -\lambda w_t Y_t -1 \leq (w_t Y_t)^2(e^{\lambda}-\lambda -1)\,,
	\]
	because the function $y \to (e^y-y-1)/y^2$ (extended by continuity at zero) is non-decreasing. Taking the expectation yields
	\[
	\EEc{e^{\lambda w_t Y_t}}{\cF_{t-1}} -1 \leq w_t^2\EEc{ Y_t^2 }{ \cF_{t-1}}(e^{\lambda}-\lambda -1)\,,
	\]
	therefore using $y+1\leq e^y$ we get
	\[
	\EEc{e^{\lambda (w_t Y_t)}}{\cF_{t-1}} \leq e^{w_t^2\EEc{ Y_t^2 }{ \cF_{t-1}}(e^{\lambda}-\lambda -1)}\,.
	\]
	We just proved that the following quantity is a supermartingale with respect to the filtration $(\cF_t)_{t\geq 0}$,
	\[
	M_t^{\lambda,+} = e^{\lambda(S_t+V_t) - V_t (e^\lambda - 1)}\,.
	\]
	Similarly, using that the same inequality holds for $-X_t$, we have
	\[
	\EEc{e^{-\lambda w_t Y_t}}{\cF_{n-1}} \leq e^{w_t^2\EEc{ Y_t^2 }{ \cF_{t-1}}(e^{\lambda}-\lambda -1)}\,,
	\]
	therefore, we can also define the supermartingale
	\[
	M_t^{\lambda,-} = e^{\lambda(-S_t+V_t) - V_t (e^\lambda - 1)}\,.
	\]
	We now choose the prior over $\lambda_x = \log(x+1)$ with $x\sim \Exponential(1)$, and consider the (mixture) supermartingale
	\[
	M_t = \frac{1}{2} \int_{0}^{+\infty} e^{\lambda_x(S_t+V_t) - V_t (e^\lambda_x - 1)} e^{-x} \mathrm{d} x +\frac{1}{2} \int_{0}^{+\infty} e^{\lambda_x(-S_t+V_t) - V_n (e^\lambda_x - 1)} e^{-x} \mathrm{d} x \,.
	\]
	Note that by construction it holds $\EE{M_t} \leq 1$. We now apply the method of mixtures to that super martingale therefore we need to lower bound it with the quantity of interest. To this aim we lower bound the integral by the one only around the maximum of the integrand. Using the change of variable $\lambda = \log(1+x)$, we obtain
	\begin{align*}
	  M_t &\geq \frac{1}{2} \int_{0}^{+\infty} e^{\lambda_x (|S_t|+V_t) - V_t (e^{\lambda_x} - 1)} e^{-x} \mathrm{d} x
	  \geq \frac{1}{2} \int_{0}^{+\infty}  e^{\lambda(|S_t| + V_t +1) - (V_t +1)(e^\lambda-1)} \mathrm{d} \lambda \\
	  &\geq  \frac{1}{2} \int_{\log\left(|S_t|/(V_t+1) + 1\right)}^{\log\left(|S_t|/(V_t+1) + 1 + 1/(V_t+1)\right)}  e^{\lambda(|S_t| + V_t +1) - (V_t +1)(e^\lambda-1)} \mathrm{d} \lambda\\
	  &\geq \frac{1}{2} \int_{\log\left(|S_t|/(V_t+1) + 1\right)}^{\log\!\left(|S_t|/(V_t+1) + 1 + 1/(V_t+1)\right)}  e^{\log\left(|S_t|/(V_t+1) + 1\right)(|S_t| + V_t +1) - |S_t| - 1} \mathrm{d} \lambda\\
	  &= \frac{1}{2e}e^{(V_t+1)h\left(|S_t|/(V_t+1)\right)} \log\left( 1+\frac{1}{|S_t|+V_t+1}\right)
	  \geq \frac{1}{4e(2 t+1)}e^{(V_t+1)h\left(|S_t|/(V_t+1)\right)}\,,
	\end{align*}
	where in the last line we used $\log(1+1/x)\geq 1/(2x)$ for $x \geq 1$ and the trivial bounds $|S_t| \leq 1$, $V_t \leq t$. The method of mixtures, see \cite{de2004self}, allows us to conclude for the first inequality of the lemma. The second inequality is a straightforward consequence of the previous one. Indeed, using that (see Exercise 2.8 of \citealp{boucheron2013concentration}) for $x\geq 0$
	  \[
	  h(x) \geq \frac{x^2}{2(1+x/3)},
	  \]
	we get
	\[
	\frac{|S_t|/b}{V_t/b^2+1} \leq \sqrt{\frac{2\log\left(4e(2t +1)/\delta\right)}{V_t/b^2+1}} + \frac{2}{3}\frac{\log\left(4e(2t+1)/\delta\right)}{V_t/b^2+1}\cdot
	\]
Multiplying by $b(V_t/b^2+1)$ the previous inequality allows us to conclude,
\begin{align*}
  |S_t| &\leq \sqrt{2(V_t+b)\log\left(4e(2t +1)/\delta\right)}+\frac{2b}{3}\log\left(4e(2t+1)/\delta\right)\\
  &\leq \sqrt{2V_t\log\left(4e(2t +1)/\delta\right)}+3b\log\left(4e(2t+1)/\delta\right).
\end{align*}
\end{proof}

\newpage
\section{Confidence intervals on values}
\label{app:optimism}
In this section, we define confidence regions for the Q-value and value functions.
To define these confidence regions, we first introduce a confidence region for the transition probabilities,
\[\cC_h^t(s,a) \triangleq \left\{q\in \Sigma_{S} :  \KL\!\big(\hp_h^{\,t}(s,a),q(s,a)\big) \leq \frac {\beta(n_h^t(s,a), \delta)} {n_h^t(s,a)}\right\}\]
for each policy $\pi$ and define the confidence regions after $t$ episodes as
\begin{align*}
 \uQ_h^{t,\pi}(s,a) &\triangleq (r_h + \max_{\up_h\in\cC_h^t(s,a)}\up_h \uV_{h+1}^{t,\pi} ) (s,a) &
   \lQ_h^{t,\pi}(s,a) &\triangleq (r_h +  \min_{\lp_h\in\cC_h^t(s,a)}\lp_h \lV_{h+1}^{t,\pi} ) (s,a) \\
 \uV_h^{t,\pi}(s) &\triangleq \pi \uQ_h^{t,\pi}(s) &  \lV_h^{t,\pi}(s) &\triangleq \pi \lQ_h^{t,\pi}(s)\\
 \uV_{H+1}^{t,\pi}(s) &\triangleq 0 &  \lV_{H+1}^{t,\pi}(s) &\triangleq 0\\
 \up_h^{t,\pi}(s,a) &\in \argmax_{\up\in\cC_h^t(s,a)} \up_h \uV_{h+1}^{t,\pi} (s,a) &
 \lp_h^{t,\pi}(s,a) &\in \argmin_{\lp\in\cC_h^t(s,a)} \lp_h \lV_{h+1}^{t,\pi} (s,a),
\end{align*}
where we use the notation $p_hf(s,a) = \E_{s'\sim p_h(.|s,a)}f(s')$ for the expectation operator and $\pi g(s) = g\big(s,\pi(s)\big)$ for the policy operator.
We also define upper and lower confidence bounds on the optimal value and Q-value functions as
\begin{align*}
 \uQ_h^{t}(s,a) &\triangleq (r_h +  \max_{\up_h\in\cC_h^t(s,a)}\up_h \uV_{h+1}^{t} ) (s,a) &
   \lQ_h^{t}(s,a) &\triangleq (r_h +  \min_{\lp_h\in\cC_h^t(s,a)}\lp_h \lV_{h+1}^{t} ) (s,a) \\
 \uV_h^{t}(s) &\triangleq \max_{a} \uQ_h^{t}(s,a) &  \lV_h^{t}(s) &\triangleq \max_{a} \lQ_h^{t}(s,a)\\
 \uV_{H+1}^{t}(s) &\triangleq 0 &  \lV_{H+1}^{t}(s) &\triangleq  0\\
 \up_h^{t}(s,a) &\in \argmax_{\up\in\cC_h^t(s,a)} \up_h \uV_{h+1}^{t} (s,a) &
 \lp_h^{t}(s,a) &\in \argmin_{\lp\in\cC_h^t(s,a)} \lp_h \lV_{h+1}^{t} (s,a)\\
 \upi_h^{t}(s,a) &\in \argmax_{a\in\cA} \uQ_{h}^{t} (s,a) &
 \lpi_h^{t}(s,a) &\in \argmax_{a\in\cA} \lQ_{h}^{t} (s,a).
\end{align*}
We now build upper confidence bounds that \emph{are not obtained by controlling the deviation of the Kullback-Leibler divergence between the empirical transition probabilities and the true transition probabilities} but only when these empirical transition probabilities are applied to the optimal value function. Precisely, similarly to \citet{azar2017minimax} and \citet{zanette2019tighter} we define upper-bounds that take into account the variance as
\begin{align*}
  \utQ_h^{t}(s,a) &\triangleq \min\Bigg(\!H, r_h(s,a) + 3\sqrt{\Var_{\hp_h^{\,t}}(\utV_{h+1}^t)(s,a) \frac{\betastar(n_h^t(s,a),\delta)}{n_h^t(s,a)}}+14H^2 \frac{\beta(n_h^t(s,a),\delta)}{n_h^t(s,a)}\\ &\qquad+\frac{1}{H}\hp_h^{\,t} (\utV_{h+1}^{t}-\ltV_{h+1}^t)(s,a)   +\hp_h^{\,t} \utV_{h+1}^{t}(s,a)\!\Bigg)\\
  \utV_h^t(s) &\triangleq \max_{a\in\cA}\utQ_h^t(s,a)\\
  \utV_{H+1}^{t}(s) &\triangleq0, \quad \text{and}\\
   \ltQ_h^{t}(s,a) &\triangleq \max\Bigg(\!0, r_h(s,a) - 3\sqrt{\Var_{\hp_h^{\,t}}(\utV_{h+1}^t)(s,a) \frac{\betastar(n_h^t(s,a),\delta)}{n_h^t(s,a)}}-14H^2 \frac{\beta(n_h^t(s,a),\delta)}{n_h^t(s,a)}\\ &\qquad-\frac{1}{H}\hp_h^{\,t} (\utV_{h+1}^{t}-\ltV_{h+1}^t)(s,a)   +\hp_h^{\,t} \ltV_{h+1}^{t}(s,a)\!\Bigg)\\
  \ltV_h^t(s) &\triangleq \max_{a\in\cA}\ltQ_h^t(s,a)\\
  \ltV_{H+1}^{t}(s) &\triangleq0.
\end{align*}
Note that, since the Bernstein inequality is not very sharp, we rather use a loose confidence bound in order to obtain simpler proof. The following confidence bounds are valid with high probability.
\begin{lemma}\label{lem:tQ_upper_bound}
  If $\betastar(n,\delta)\leq \beta(n,\delta)$ then on event $\cG,$ we have that for all $t$, all $h\in[H]$, and all $(s,a),$
  \begin{align*}
    \ltQ_h^t(s,a) &\leq \Qstar_h(s,a) \leq \utQ_h^t(s,a) \quad \text{and} \\
     \ltV_h^t(s) &\leq \Vstar_h(s) \leq \utV_h^t(s) .
  \end{align*}
\end{lemma}
\begin{proof}
We proceed by induction. For $h=H+1$ the result is trivially true. Assume the inequalities hold for $h'>h$. We prove only the upper bounds, the proof is very similar for the lower bounds. Fix $(s,a)$ and assume that $\utQ_h^t(s,a)<H,$ otherwise the we trivially have that  $\utQ_h^t(s,a)\geq \Qstar_h(s,a)$. Note that it implies that $n_h^t(s,a)>0$. In such case, by construction we have that
  \begin{align}
  \utQ_h(s,a) - \Qstar_h(s,a) &\geq 3\sqrt{\Var_{\hp_h^{\,t}}(\utV_{h+1}^t)(s,a) \frac{\betastar(n_h^t(s,a),\delta)}{n_h^t(s,a)}}+ 10 H^2 \frac{\beta(n_h^t(s,a),\delta)}{n_h^t(s,a)}\nonumber \\
  &\qquad +\frac{1}{H}\hp_h^{\,t} (\utV_{h+1}^{t}-\ltV_{h+1}^t)(s,a) + (\hp_h^{\,t}-p_h) \Vstar_{h+1}(s,a)\nonumber\\
  &\qquad + \hp_h^{\,t} (\utV_{h+1}^{t}-\Vstar_{h+1})(s,a)\label{eq:tq_minus_qstar_bpi}.
  \end{align}
 Then, by definition of event $\cG$ and in particular $\cE^\star,$ we know that
  \begin{align*}
    \big|(\hp_h^{\,t}-p_h) \Vstar_{h+1}(s,a)\big| \leq \sqrt{2\Var_{p_h}(\Vstar_{h+1})(s,a)\frac{\betastar(n_h^t(s,a),\delta) }{n_h^t(s,a)}}+3 H \frac{\betastar(n_h^t(s,a),\delta) }{n_h^t(s,a)}\cdot
  \end{align*}
  Using Lemma~\ref{lem:switch_variance_bis} and Lemma~\ref{lem:switch_variance}, the definition of $\cG$ and the induction hypothesis, we get
    \begin{align*}
    \Var_{p_h}(\Vstar_{h+1})(s,a) &\leq 2\Var_{\hp_h^{\,t}}(\Vstar_{h+1})(s,a) + 4 H^2 \frac{\beta(n_h^t(s,a),\delta)}{n_h^t(s,a)}\\
    &\leq  4\Var_{\hp_h^{\,t}}(\utV_{h+1}^t)(s,a) + 4 H \hp_h^{\,t}(\utV_{h+1}^t- \Vstar_{h+1})(s,a) + 4H^2 \frac{\beta(n_h^t(s,a),\delta)}{n_h^t(s,a)}\\
    &\leq 4\Var_{\hp_h^{\,t}}(\utV_{h+1}^t)(s,a) + 4 H \hp_h^{\,t} (\utV_{h+1}^t-\ltV_{h+1}^t)(s,a) + 4H^2 \frac{\beta(n_h^t(s,a),\delta)}{n_h^t(s,a)}\cdot
  \end{align*}
Therefore using $\betastar(n,\delta) \leq \beta(n,\delta)$, $\sqrt{x+y}\leq \sqrt{x}+\sqrt{y}$ and $\sqrt{xy}\leq x+y$ for $x,y\geq 0$, we obtain
\begin{align*}
  \big|(\hp_h^{\,t}-p_h) \Vstar_{h+1}(s,a)\big| &\leq 2\sqrt{2}\sqrt{\Var_{\hp_h^{\,t}}(\utV_{h+1}^t)(s,a)\frac{\betastar(n_h^t(s,a),\delta) }{n_h^t(s,a)}} \\
  &\qquad+ \sqrt{\frac{1}{H}\hp_h^{\,t}(\utV_{h+1}^t-\ltV_{h+1}^t)(s,a)  8H^2\frac{\beta(n_h^t(s,a),\delta) }{n_h^t(s,a)}}+ (\sqrt{8}+3) H^2 \frac{\beta(n_h^t(s,a),\delta)}{n_h^t(s,a)}\\
  &\leq 3\sqrt{\Var_{\hp_h^{\,t}}(\utV_{h+1}^t)(s,a)\frac{\betastar(n_h^t(s,a),\delta) }{n_h^t(s,a)}} + 14 H^2 \frac{\beta(n_h^t(s,a),\delta)}{n_h^t(s,a)}\\
  &\qquad + \frac{1}{H} \hp_h^{\,t}(\utV_{h+1}^t-\ltV_{h+1}^t)(s,a) .
\end{align*}
Now going back to~\eqref{eq:tq_minus_qstar_bpi}, due to the inequality above, we arrive to
\begin{align*}
\utQ_h(s,a) - \Qstar_h(s,a) \geq  \hp_h^{\,t} (\utV_{h+1}^{t}-\Vstar_{h+1})(s,a) \geq 0,
\end{align*}
where we used the induction hypothesis. Therefore, we have that for all $(s,a)$, $\utQ_h^t(s,a)\geq \Qstar_h(s,a)$ and then by definition $\utV_h^t(s) = \max_{a}\utQ_h^t(s,a) \geq \Vstar_h(s)$.
\end{proof}

\newpage
\section{Proofs of reward-free exploration results}
\label{app:proofs_RF}

\lemmaWRF*
\begin{proof}
Assume that we are on event $\cF$ and we fix a policy $\pi$.

\paragraph{Step 1: Empirical Bernstein inequality} From Lemma~\ref{lem:Bernstein_via_kl} we know that for any state-action pair $(s,a),$ if $n_h^t(s,a)>0$, then {\small
\begin{align}
  \he_h^{\,t,\pi}(s,a; r) &= \big|\hQ_h^{t,\pi}(s,a; r)-Q_h^\pi(s,a; r)\big| = \big|(\hp_h^{\,t}-p_h)V_{h+1}^{\pi}(s,a ; r)\big|+\hp_h^{\,t}|\hV^{t,\pi}_{h+1}-V_{h+1}^\pi|(s,a)\nonumber\\
  &\leq \sqrt{2 \Var_{p_h}(V_{h+1}^\pi)(s,a ; r) \frac{\beta(n_h^t(s,a),\delta)}{n_h^t(s,a)}}+\frac{2}{3}H\frac{\beta(n_h^t(s,a),\delta)}{n_h^t(s,a)} +   \hp_h^{\,t}(\hV_{h+1}^{t,\pi}-V_{h+1}^\pi)(s,a)\label{eq:initial_bound_H3}.
\end{align}}
\hspace{-0.2cm}Now notice that by Lemma~\ref{lem:switch_variance_bis} with the fact that $0\leq V_{h+1}^\pi \leq H$
\begin{align*}
  \Var_{p_h}(V_{h+1}^\pi)(s,a; r) &\leq 2\Var_{\hp_h^{\,t}}(V_{h+1}^\pi)(s,a; r) +4 H^2\KL(\hp_h^{\,t}(s,a),p_h(s,a))\\
  &\leq 2\Var_{\hp_h^{\,t}}(V_{h+1}^\pi)(s,a; r) +4H^2\frac{\beta(n_h^t(s,a),\delta)}{n_h^t(s,a)}\cdot
\end{align*}
Using Lemma~\ref{lem:switch_variance}, we replace the value function by its estimate to get
\[
\Var_{\hp_h^{\,t}}(V_{h+1}^\pi)(s,a; r)  \leq 2 \Var_{\hp_h^{\,t}}(\hV_{h+1}^{t,\pi})(s,a; r)+2H\hp_{h}^{\,t}|V_{h+1}^\pi -\hV_{h+1}^{t,\pi}|(s,a).
\]
Combining these two inequalities and using $\sqrt{xy} \leq x+y$ for $x,y\geq 0$, we arrive to
\begin{align*}
  \sqrt{2 \Var_{p_h}(V_{h+1}^\pi)(s,a; r) \frac{\beta(n_h^t(s,a),\delta)}{n_h^t(s,a)}} &\leq 2\sqrt{2}\sqrt{ \Var_{\hp_h^{\,t}}(\hV_{h+1}^{t,\pi})(s,a; r) \frac{\beta(n_h^t(s,a),\delta)}{n_h^t(s,a)}}+\sqrt{8}H \frac{\beta(n_h^t(s,a),\delta)}{n_h^t(s,a)}\\
  &\qquad+ \sqrt{\frac{1}{H} \hp_h^{\,t}|V_{h+1}^\pi -\hV_{h+1}^{t,\pi}|(s,a) 8H^2\frac{\beta(n_h^t(s,a),\delta)}{n_h^t(s,a)}}\\
  &\leq 3\sqrt{ \Var_{\hp_h^{\,t}}(\hV_{h+1}^{t,\pi})(s,a; r) \frac{\beta(n_h^t(s,a),\delta)}{n_h^t(s,a)}} \\
  &
  \qquad+ (\sqrt{8}H+8H^2)\frac{\beta(n_h^t(s,a),\delta)}{n_h^t(s,a)}
  + \frac{1}{H}\hp_{h}^{\,t}|V_{h+1}^\pi -\hV_{h+1}^{t,\pi}|(s,a).
\end{align*}
Injecting the above inequality in the initial bound~\eqref{eq:initial_bound_H3} of the estimation error yields
\begin{align*}
  \he_h^{t,\pi}(s,a; r) &\leq 3\sqrt{ \frac{\Var_{\hp_h^{\,t}}(\hV_{h+1}^{t,\pi})(s,a; r)}{H^2} \frac{H^2\beta(n_h^t(s,a),\delta)}{n_h^t(s,a)}} + 12H^2\frac{\beta(n_h^t(s,a),\delta)}{n_h^t(s,a)}\\
  &\qquad\qquad+\left(1+\frac{1}{H}\right)\hp_h^{\,t} |V_{h+1}^\pi -\hV_{h+1}^{t,\pi}|(s,a)\\
  &\leq  3\sqrt{ \frac{\Var_{\hp_h^{\,t}}(\hV_{h+1}^{t,\pi})(s,a; r)}{H^2} \left(\frac{H^2\beta(n_h^t(s,a),\delta)}{n_h^t(s,a)}\wedge 1\right)} +15H^2\frac{\beta(n_h^t(s,a),\delta)}{n_h^t(s,a)}\\
  &\qquad\qquad+\left(1+\frac{1}{H}\right)\hp_h^{\,t} |V_{h+1}^\pi -\hV_{h+1}^{t,\pi}|(s,a),
\end{align*}
where in the last inequality we used that if $H^2\beta(n_h^t(s,a),\delta)/n_h^t(s,a)\geq 1$ then
\[
3\sqrt{ \frac{\Var_{\hp_h^{\,t}}(\hV_{h+1}^{t,\pi})(s,a; r)}{H^2} \frac{H^2\beta(n_h^t(s,a),\delta)}{n_h^t(s,a)}} \leq 3\sqrt{\frac{H^2\beta(n_h^t(s,a),\delta)}{n_h^t(s,a)}} \leq 3 \frac{H^2\beta(n_h^t(s,a),\delta)}{n_h^t(s,a)}\cdot
\]
We therefore obtain the following bound on the error of estimation of the Q-value function at a state-action pair $(s,a)$ in the case when $n_h^t(s,a)>0$,
\begin{align*}
  \he_h^{t,\pi}(s,a; r) &\leq 3\sqrt{ \frac{\Var_{\hp_h^{\,t}}(\hV_{h+1}^{t,\pi})(s,a; r)}{H^2} \left(\frac{H^2\beta(n_h^t(s,a),\delta)}{n_h^t(s,a)}\wedge 1\right)} +15H^2\frac{\beta(n_h^t(s,a),\delta)}{n_h^t(s,a)}\\
  &\qquad\qquad+\left(1+\frac{1}{H}\right)\hp_h^{\,t} \pi_{h+1} \he_{h+1}^{t,\pi}(s,a; r),
\end{align*}
where we recall the operator notation $\pi_{h} g (s') \triangleq g(s',\pi_h(s'))$. Defining recursively the functions $Z$ as $Z_{H+1}^{t,\pi}(s,a; r) \triangleq 0$ and for $h\leq h$ as
\begin{align*}
  Z_h^{t,\pi}(s,a; r) &\triangleq \min\Bigg(\!H, 3\sqrt{ \frac{\Var_{\hp_h^{\,t}}(\hV_{h+1}^{t,\pi})(s,a; r)}{H^2} \left(\frac{H^2\beta(n_h^t(s,a),\delta)}{n_h^t(s,a)}\wedge 1\right)} +15H^2\frac{\beta(n_h^t(s,a),\delta)}{n_h^t(s,a)}\\
 &\qquad\qquad+\left(1+\frac{1}{H}\right)\hp_h^{\,t} \pi_{h+1} Z_{h+1}^{t,\pi}(s,a; r)\!\Bigg)
\end{align*}
and noting that $  \he_h^{\,t,\pi}(s,a; r) \leq H$ we can
show by induction that for all $(s,a,h),$
\begin{equation}
  \label{eq:he_lower_Z_H3}
\he_h^{\,t,\pi}(s,a; r)\leq Z_h^{t,\pi}(s,a; r).
\end{equation}

\paragraph{Step 2: Law of total variance} For all $(s,a),$ we now recursively define $Y$ and $W$. In particular, we set $Y_{H+1}^{t,\pi}(s,a; r) \triangleq W_{H+1}^{t,\pi}(s,a) \triangleq 0$ and
\begin{align*}
  Y_h^{t,\pi}(s,a; r) & \triangleq  3\sqrt{ \frac{\Var_{\hp_h^{\,t}}(\hV_{h+1}^{t,\pi})(s,a; r)}{H^2} \left(\frac{H^2\beta(n_h^t(s,a),\delta)}{n_h^t(s,a)}\wedge 1\right)} + \left(1+\frac{1}{H}\right)\hp_h^{\,t} \pi_{h+1} Y_{h+1}^{t,\pi}(s,a; r)\\
  W_h^{t,\pi}(s,a) &\triangleq \min\left(\!H, 15H^2 \frac{\beta(n_h^t(s,a),\delta)}{n_h^t(s,a)} +\left(1+\frac{1}{H}\right)\hp_h^{\,t} \pi_{h+1} W_{h+1}^{t,\pi}(s,a) \! \right).
\end{align*}
Again by induction we show that for all $h,s,a$
\[
Z_h^{t,\pi}(s,a; r) \leq Y_h^{t,\pi}(s,a; r)+W_h^{t,\pi}(s,a).
\]
Indeed, the case $h=H+1$ is trivially true and if we assume the inequality is true at step $h+1$ then using that $\min(x,y+z)\leq \min(x,y)+\min(x,z)$ for $x,y,z\geq 0$,
\begin{align*}
  Z_h^{t,\pi}(s,a; r) &\leq  \min\Bigg(H, 3\sqrt{ \frac{\Var_{\hp_h^{\,t}}(\hV_{h+1}^{t,\pi})(s,a; r)}{H^2} \left(\frac{H^2\beta(n_h^t(s,a),\delta)}{n_h^t(s,a)}\wedge 1\right)} +15H^2\frac{\beta(n_h^t(s,a),\delta)}{n_h^t(s,a)}\\
 &\qquad\qquad+\left(1+\frac{1}{H}\right)\hp_h^{\,t} \pi_{h+1} Y_{h+1}^{t,\pi}(s,a; r)+\left(1+\frac{1}{H}\right)\hp_h^{\,t} \pi_{h+1} W_{h+1}^{t,\pi}(s,a)\Bigg)\\
 &\leq  Y_h^{t,\pi}(s,a; r)+W_h^{t,\pi}(s,a).
\end{align*}
Now using~\eqref{eq:he_lower_Z_H3} we have
\begin{align}\label{eq:he_lower_YpW_H3}
\pi_1 \he_1^{\,t,\pi}(s_1; r) \leq \pi_1 Y_1^{t,\pi}(s_1; r)+\pi_1 W_1^{t,\pi}(s_1).
\end{align}
Next, we upper-bound the $Y_h^{t,\pi}$ term to \emph{remove the dependency on the empirical variance} of the the value function the policy $\pi$.
We let $\hp_h^{\,t,\pi}(s,a)$ be the probability of reaching the sate-action pair $(s,a)$ at step $h$ under policy $\pi$ with the empirical transitions at round $t$. By successive use of the definition of $Y_h^{t,\pi}$, the Cauchy-Schwarz inequality, and Lemma~\ref{lem:law_of_total_variance} in the empirical MDP, we get
\begin{align*}
  \pi Y_1^{t,\pi}(s_1; r) &= 3\sum_{s,a}\sum_{h=1}^H \hp_h^{t,\pi}(s,a) \left(1+\frac{1}{H}\right)^{h-1}  \sqrt{ \frac{\Var_{\hp_h^{\,t}}(\hV_{h+1}^{t,\pi})(s,a; r)}{H^2} \left(\frac{H^2\beta(n_h^t(s,a),\delta)}{n_h^t(s,a)}\wedge 1\right)}\\
  &\leq 3e \sqrt{\sum_{s,a}\sum_{h=1}^H \hp_h^{t,\pi}(s,a)\frac{ \Var_{\hp_h^{\,t}}(\hV_{h+1}^{t,\pi})(s,a; r)}{H^2}}\sqrt{\sum_{s,a}\sum_{h=1}^H \hp_h^{t,\pi}(s,a) \left(\frac{H^2\beta(n_h^t(s,a),\delta)}{n_h^t(s,a)}\wedge 1\right)}\\
  &\leq 3e \sqrt{\frac{1}{H^2}  \E_{\pi, \hp_h^{\,t}}\left[ \left(\sum_{h=1}^H r_h(s_h,a_h) -\hV_1^{\pi}(s_1; r)\right)^2\right]  }\sqrt{\sum_{s,a}\sum_{h=1}^H \hp_h^{t,\pi}(s,a) \left(\frac{H^2\beta(n_h^t(s,a),\delta)}{n_h^t(s,a)}\wedge 1\right)}\\
  &\leq 3e\sqrt{\sum_{s,a}\sum_{h=1}^H \hp_h^{t,\pi}(s,a) \left(\frac{H^2\beta(n_h^t(s,a),\delta)}{n_h^t(s,a)}\wedge 1\right)}.
\end{align*}

\paragraph{Step 3: Clipping} For this step we recursively define $\tW^{t,\pi}$ by  $\tW_{H+1}^{t,\pi}(s,a) \triangleq 0$ and
\[
\tW_{h}^{t,\pi}(s,a) \triangleq \left(\frac{H^2\beta(n_h^t(s,a),\delta)}{n_h^t(s,a)}\wedge 1\right) + \hp_h^{t,\pi} \pi_{h+1} \tW_{h+1}^{t,\pi}(s,a),
\]
such that by construction
\[
\sum_{s,a}\sum_{h=1}^H \hp_h^{t,\pi}(s,a) \left(\frac{H^2\beta(n_h^t(s,a),\delta)}{n_h^t(s,a)}\wedge 1\right) = \pi_1 \tW_1^{t,\pi}(s_1).
\]
By induction we can prove that $\tW_h^{t,\pi}\leq W_h^{t,\pi}$. Indeed, the inequality is true for $h=H+1$ and if we assume it is true for step $h+1$ then using that by construction $\tW_{h}^{t,\pi}(s,a)\leq H$, for all $(s,a)$,
\begin{align*}
  \tW_{h}^{t,\pi}(s,a) &=\min\left(H, \left(\frac{H^2\beta(n_h^t(s,a),\delta)}{n_h^t(s,a)}\wedge 1\right) + \hp_h^{t,\pi} \pi_{h+1} \tW_{h+1}^{t,\pi}(s,a)\right)\\
  &\leq \min\left(H, \frac{H^2\beta(n_h^t(s,a),\delta)}{n_h^t(s,a)} + \hp_h^{t,\pi} \pi_{h+1} W_{h+1}^{t,\pi}(s,a)\right)\\
  &\leq W_{h}^{t,\pi}(s,a).
\end{align*}
Therefore, we just proved that $ \pi_1 Y_1^{t,\pi}(s_1; r)\leq 3e \sqrt{\pi_1 W_1^{t,\pi}(s_1)}$ and going back to~\eqref{eq:he_lower_YpW_H3} we get
\begin{align}
  \label{eq:he_lower_W_H3}
  \pi_1 \he_1^{t,\pi}(s_1; r) \leq  3e \sqrt{\pi_1 W_1^{t,\pi} (s_1)} + \pi_1 W_1^{t,\pi}(s_1).
\end{align}

\paragraph{Conclusion} To finish the proof it remains to note that for all $\pi$, for all $(s,a,h)$ we have that
\[
W^{t,\pi}_h(s,a) \leq W^t_h(s,a)\qquad
\text{and}
\qquad
 \pi_h W_h^{t,\pi}(s) \leq \max_{a\in\cA}W^t_h(s) = \pi_{h}^{t+1}W^t_h(s,a).
\]
\end{proof}

\theoremRF*

\begin{proof}
We first prove that \OurAlgorithmRF is $(\epsilon,\delta)$-PAC. This is a simple consequence of Lemma~\ref{lem:upper_bound_error_W}. Indeed if the algorithm stops at round $\tau$ we know that on  event $\cF$ for any policy $\pi,$
\begin{align*}
  \frac{\epsilon}{2} &\geq 2e \sqrt{\max_{a\in\cA} W_1^{\tau} (s_1,a)} + \max_{a\in\cA} W_1^{\tau}(s_1,a)
  \geq \pi_1 \he_1^{\,\tau,\pi}(s_1; r) = | \hV^{\pi,\tau}_1(s_1; r) - V_1^\pi(s_1; r)|.
\end{align*}
Recalling with are still on event $\cF,$ we have
\begin{align*}
V_1^{\star}(s_1; r) - V_1^{\hpi^{\star,\tau}}(s_1; r) &= V_1^{\star}(s_1; r) - \hV_1^{\tau, \pistar}(s_1; r) + \hV_1^{\tau,\pistar}(s_1; r) - \hV_1^{\tau, \hpi^{\star,\tau}}(s_1; r) \\ &\quad + \hV_1^{\tau, \hpi^{\star,\tau}}(s_1; r)- V_1^{\hpi^{\star,\tau}}(s_1; r) \\
&\leq |V_1^{\star}(s_1; r) - \hV_1^{\tau, \pistar}(s_1; r)| + |\hV_1^{\tau, \hpi^{\star,\tau}}(s_1; r)- V_1^{\hpi^{\star,\tau}}(s_1; r)| \leq \epsilon.
\end{align*}
We can conclude the first part of the theorem by noting that by Lemma~\ref{lem:proba_master_event}, $\P(\cF)\geq 1-\delta$.

\medskip
\noindent
It remains to prove the \emph{bound on the sample complexity}. In the rest of the proof we again assume that event $\cF$ holds. We start by fixing a round $T<\tau$.

\paragraph{Step 1: Upper bound on $W_1^t$} We first provide an upper bound on $W_h^t(s,a)$ for all $(s,a,h)$ and $t\leq T$. By definition, if $n_h^t(s,a)>0,$  then
\begin{align*}
  W_h^t(s,a) &\leq 15H^2 \frac{\beta(n_h^t(s,a),\delta)}{n_h^t(s,a)} +\left(1+\frac{1}{H}\right)\sum_{s'}\hp_h^{\,t}(s'|s,a) \max_{a'} W_{h+1}^{t}(s',a')\\
  &=15H^2\frac{\beta(n_h^t(s,a),\delta)}{n_h^t(s,a)} +\left(1+\frac{1}{H}\right)(\hp_h^{\,t}-p_h) \pi_{h+1}^{t+1} W_{h+1}^{t}(s,a)\\
  &\qquad+\left(1+\frac{1}{H}\right) p_h \pi_{h+1}^{t+1} W_{h+1}^{t}(s,a).
\end{align*}
Using Lemma~\ref{lem:Bernstein_via_kl} we apply Bernstein inequality to get
\begin{align*}
(\hp_h^{\,t}-p_h) \pi_{h+1}^{t+1} W_{h+1}^{t}(s,a) \leq \sqrt{2\Var_{p_h}(\pi_{h+1}^{t+1}W_{h+1}^{t})(s,a)  \frac{\beta(n_h^t(s,a),\delta)}{n_h^t(s,a)} } + \frac{2}{3}H \frac{\beta(n_h^t(s,a),\delta)}{n_h^t(s,a)}\cdot
\end{align*}
Now since the variance is bounded by $\Var_{p_h}(\pi_{h+1}^{t+1}W_{h+1}^{t})(s,a) \leq H p_h  \pi_{h+1}^{t+1}W_{h+1}^{t} (s,a), $ we use the fact that $\sqrt{x y}\leq x+y$ for $x,y\geq 0$ and split the square-root term into two other terms
\begin{align*}
  \sqrt{2\Var_{p_h}(\pi_{h+1}^{t+1}W_{h+1}^{t})(s,a)  \frac{\beta(n_h^t(s,a),\delta)}{n_h^t(s,a)} }  &\leq \sqrt{\frac{1}{H} p_h  \pi_{h+1}^{t+1}W_{h+1}^{t} (s,a) 2H^2 \frac{\beta(n_h^t(s,a),\delta)}{n_h^t(s,a)}}\\
  &\leq \frac{1}{H} p_h  \pi_{h+1}^{t+1}W_{h+1}^{t} (s,a)  + 2H^2 \frac{\beta(n_h^t(s,a),\delta)}{n_h^t(s,a)}\cdot
\end{align*}
We plug this upper bound in the Bernstein inequality above to obtain
 \begin{align*}
   \left(1+\frac{1}{H}\right)(\hp_h^{\,t}-p_h) \pi_{h+1}^{t+1} W_{h+1}^{t}(s,a) \leq  \frac{16}{3} H^2 \frac{\beta(n_h^t(s,a),\delta)}{n_h^t(s,a)} + \frac{2}{H} p_h  \pi_{h+1}^{t+1}W_{h+1}^{t} (s,a).
 \end{align*}
 Next, we go back to the initial upper bound and note that by  our construction, $W_h^t(s,a)\leq H\leq H^2$  we get for all $n_h^t(s,a)\geq 0$,
\begin{equation}
  \label{eq:upper_bound_Wt}
  W_h^t(s,a) \leq 21 H^2 \left(\frac{\beta(n_h^t(s,a),\delta)}{n_h^t(s,a)}\wedge 1\right) +\left(1+\frac{3}{H}\right)p_h \pi_{h+1}^{t+1} W_{h+1}^{t}(s,a).
\end{equation}
Unfolding~\eqref{eq:upper_bound_Wt} and using that $(1+3/H)^H\leq e^3$ yields
\begin{equation*}
  \pi_1^{t+1} W_1^t(s_1) \leq 21 e^3 H^2 \sum_{h=1}^H \sum_{s,a} p_h^{t+1}(s,a)
  \left(\frac{\beta(n_h^t(s,a),\delta)}{n_h^t(s,a)}\wedge 1\right).
\end{equation*}
Using Lemma~\ref{lem:cnt_pseudo} we replace the counts by the pseudo-counts in the above inequality
  \begin{equation}
    \label{eq:upper_bound_W1}
    \pi_1^{t+1} W_1^t(s_1) \leq 84 e^3 H^2\sum_{h=1}^H \sum_{s,a} p_h^{t+1}(s,a)
    \frac{\beta(\bn_h^t(s,a),\delta)}{\bn_h^t(s,a)\vee 1}\CommaBin
  \end{equation}
where we recall that $n_h^t(s,a) \triangleq \sum_{\ell = 1}^t p_h^\ell(s,a)$.

\paragraph{Step 2: Summing over $t\leq T$}
Since $T<\tau,$ we know that for $t\leq T$ due to the stopping rule,
\[
 \epsilon \leq 3e \sqrt{\pi_1^{t+1} W_1^{t} (s_1)} + \pi_1^{t+1} W_1^{t}(s_1).
\]
Summing the over the above inequalities for $0\leq t\leq T, $ followed by Cauchy-Schwarz inequality,
\begin{align*}
(T+1)\epsilon &\leq \sum_{t=0}^T 3e \sqrt{\pi_1^{t+1} W_1^{t} (s_1)} + \pi_1^{t+1} W_1^{t}(s_1) \\
&\leq 3e \sqrt{(T+1) \sum_{t=0}^T \pi_1^{t+1} W_1^{t} (s_1) } +  \sum_{t=0}^T \pi_1^{t+1} W_1^{t} (s_1).
\end{align*}
Next,  we upper-bound the sum that appears in the left-hand terms. Using successively~\eqref{eq:upper_bound_W1}, the property that $\beta(\cdot,\delta)$ is increasing and Lemma~\ref{lem:sum_1_over_n}, we have
\begin{align*}
   \sum_{t=0}^T \pi_1^{t+1} W_1^{t} (s_1) & \leq  84 e^3 H^2\sum_{t=0}^T\sum_{h=1}^H \sum_{s,a} p_h^{t+1}(s,a)
   \frac{\beta(\bn_h^t(s,a),\delta)}{\bn_h^t(s,a)\vee 1}\\
   &\leq 84 e^3 H^2 \beta(T,\delta) \sum_{t=0}^T\sum_{h=1}^H \sum_{s,a} p_h^{t+1}(s,a)
   \frac{1}{\bn_h^t(s,a)\vee 1}\\
   &= 84 e^3 H^2 \beta(T,\delta)\sum_{h=1}^H \sum_{s,a} \sum_{t=0}^T
   \frac{\bn_h^{t+1}(s,a) - \bn_h^t(s,a)}{\bn_h^t(s,a)\vee 1}\\
   &\leq 336 e^3 H^3 S A \log(T+2) \beta(T,\delta).
\end{align*}
Therefore, combining just obtained inequality with the previous one, we get
\[
(T+1)\epsilon \leq 55e^3 \sqrt{(T+1) H^3 S A \log(T+2) \beta(T,\delta)} +   336 e^3 H^3 S A \log(T+2) \beta(T,\delta).
\]
We now assume that $\tau>0$ otherwise the result is trivially true. Since the above inequality is true for all $T<\tau,$ we get the functional inequality on $\tau$
\begin{equation}
  \label{eq:bound_tau_functional}
  \epsilon \tau \leq 55e^3 \sqrt{\tau H^3 S A \log(\tau+1) \beta(\tau-1,\delta)} +   336 e^3 H^3 S A \log(\tau+1) \beta(\tau-1,\delta).
\end{equation}
\textbf{Step 3: Upper bound on $\tau$} It remains to invert of~\eqref{eq:bound_tau_functional}. Due the the specific choice of $\beta$ we are able to further upper-bound $\tau$ by
\begin{align*}
  \epsilon \tau &\leq 55e^3 \sqrt{\tau H^3 S A \left( \log(3SAH/\delta) \log(8e\tau) + S\log(8e\tau)^2\right)} \\
  &\qquad+   336 e^3 H^3 S A \left( \log(3SAH/\delta) \log(8e\tau) + S\log(8e\tau)^2\right).
\end{align*}
We are now ready to use Lemma~\ref{lem:reverse_inequality} with
\[ C = 55e^3 \sqrt{H^3SA}/\epsilon,\  A = \log(3SAH/\delta),\  B=E=S,\  D = 336 e^3 H^3 S A /\epsilon,\   \text{and}\    \alpha = 8e \]
to get
\[
\tau \leq \frac{H^3SA}{\epsilon^2} \left( \log(3SAH/\delta) +S\right) C_1+1,
\]
where we chose $C_1 \triangleq 5587 e^6 \log\left(e^{18} \left( \log(3SAH/\delta) +S\right) H^3SA/\epsilon\right)^2$ and used that $\epsilon\leq 1$.

\end{proof}

\newpage
\section{Proofs of best-policy identification results}
\label{app:proof_BPI}
For this section, we  recursively define  upper bound on the gap $\Vstar_1(s_1)-V_1^{\pi^{t+1}}(s_1)$ in order to build the stopping rule. Precisely consider $G_{H+1}^t(s,a) \triangleq 0$ for all $(s,a)$ and for all $(s,a,h)$,
\[
  G_h^{\,t}(s,a) \triangleq \min\left(\!H,  6\sqrt{\Var_{\hp_h^{\,t}}(\utV_{h+1}^t)(s,a) \frac{\betastar(n_h^t(s,a),\delta)}{n_h^t(s,a)}}+ 36 H^2 \frac{\beta(n_h^t(s,a),\delta)}{n_h^t(s,a)}
  + \left(1+\frac{3}{H}\right)\hp_h^{\,t} \pi_{h+1}^{t+1} G_{h+1}^t(s)\right).
\]
\lemmaBPI*
\begin{proof}
First, we define the following quantities
\begin{align*}
  \rQ_h^t(s,a) &\triangleq \min\Bigg(\! r_h(s,a) + p_h \rV_h^t(s,a), \max\Bigg(0, r_h(s,a)  - 3\sqrt{\Var_{\hp_h^{\,t}}(\utV_{h+1}^t)(s,a) \frac{\betastar(n_h^t(s,a),\delta)}{n_h^t(s,a)}}\\
  &\qquad-14H^2 \frac{\beta(n_h^t(s,a),\delta)}{n_h^t(s,a)}-\frac{1}{H}\hp_h^{\,t} (\utV_{h+1}^{t}-\ltV_{h+1}^t)(s,a)   +\hp_h^{\,t} \rV_{h+1}^{t}(s,a)\!\Bigg)  \!\Bigg)  \\
  \rV_h^t(s) &\triangleq\pi^{t+1}_h\rQ_h^t(s,a) \\
  \rV_{H+1}^t (s) &\triangleq 0.
\end{align*}
On the event $\cG,$ using Lemma~\ref{lem:lower_bound_rQ_BPI} below, we upper-bound the gap at time $t$ as
\[
\Vstar_1\left(s_1\right)-V^{\pi^{t+1}}_1\left(s_1\right) \leq \utV_1^t\left(s_1\right) - V^{\pi^{t+1}}_1\left(s_1\right) \leq  \utV_1^t\left(s_1\right) - \rV_1^t\left(s_1\right).
\]
Next, we upper-bound the last term in the bound above. We do it by induction on $h$, showing that for all state-action pairs $(s,a)$,
\begin{align}
  \label{eq:G_lager_gap_bounds}
 \utQ_h^t(s,a) - \rQ_h^t(s,a) \leq G_h^{\,t}(s,a).
\end{align}
The inequality is trivially true for $h=H+1$. For the induction step, assume it is true for $h+1$ and fix a state-action pair $(s,a)$. Since by construction $\utQ_h^t\leq H$ and $\rQ_h^t \geq 0,$ then if $G_h^{\,t}(s,a)=H$ the inequality is trivially true. We therefore assume that $G_h^{\,t}(s,a)< H$ which implies in particular that $n_h^t(s,a)>0$. We now distinguish the two possible values of $\rQ_h^t(s,a)$.
\paragraph{Step 1: First case} If $\rQ_h^t(s,a) = r_h(s,a) + p_h \rV_h^t(s,a)$, then we have
\begin{align*}
  \utQ_h^t(s,a)-\rQ_h^t(s,a) &\leq 3\sqrt{\Var_{\hp_h^{\,t}}(\utV_{h+1}^t)(s,a) \frac{\betastar(n_h^t(s,a),\delta)}{n_h^t(s,a)}}+14H^2 \frac{\beta(n_h^t(s,a),\delta)}{n_h^t(s,a)}\\ &\qquad+\frac{1}{H}\hp_h^{\,t} (\utV_{h+1}^{t}-\ltV_{h+1}^t)(s,a)   +\hp_h^{\,t} \utV_{h+1}^{t}(s,a) -p_h \rV_{h+1}^t(s,a)\,.
\end{align*}
We upper-bound  the last term separately as
\begin{align}
\nonumber
\hp_h^{\,t} \utV_{h+1}^{t}(s,a) -p_h \rV_{h+1}^{t}(s,a) &= \hp_h^{\,t}(\utV_{h+1}^t -\rV_{h+1}^{t})(s,a) + (\hp_h^{\,t}-p_h)\Vstar_{h+1}(s,a) \\
& \quad + (p_h-\hp_h^{\,t})(\Vstar_{h+1}-\rV_{h+1}^{t})(s,a).   \label{eq:last_term_gap_BPI}
\end{align}
First, proceeding as in the proof of Lemma~\ref{lem:tQ_upper_bound} we know that
\begin{align*}
  (\hp_h^{\,t}-p_h)\Vstar_{h+1}(s,a) &\leq 3\sqrt{\Var_{\hp_h^{\,t}}(\utV_{h+1}^t)(s,a)\frac{\betastar(n_h^t(s,a),\delta) }{n_h^t(s,a)}} + 14 H^2 \frac{\beta(n_h^t(s,a),\delta)}{n_h^t(s,a)}\\
  &\qquad + \frac{1}{H} \hp_h^{\,t}\left(\utV_{h+1}^t-\ltV_{h+1}^t\right)(s,a).
\end{align*}
Second, using Lemma~\ref{lem:Bernstein_via_kl} yields
\begin{align*}
  (p_h-\hp_h^{\,t})(\Vstar_{h+1}-\rV_{h+1}^{t})(s,a) \leq \sqrt{2\Var_{p_h}(\Vstar_{h+1}-\rV_{h+1}^{t})(s,a)\frac{\beta(n_h^t(s,a),\delta) }{n_h^t(s,a)}} +  \frac{2}{3}H \frac{\beta(n_h^t(s,a),\delta)}{n_h^t(s,a)}\cdot
\end{align*}
Using Lemma~\ref{lem:switch_variance_bis}, we simplify the variance term that appears in the inequality above as
\begin{align*}
  \Var_{p_h}(\Vstar_{h+1}-\rV_{h+1}^{t})(s,a) &\leq 2\Var_{\hp_h^{\,t}}(\Vstar_{h+1}-\rV_{h+1}^{t})(s,a) + 4H^2\frac{\beta(n_h^t(s,a),\delta)}{n_h^t(s,a)} \\
  &\leq 2 H \hp_h^{\,t} (\Vstar_{h+1}-\rV_{h+1}^{t})(s,a) +4H^2 \frac{\beta(n_h^t(s,a),\delta)}{n_h^t(s,a)}\cdot
\end{align*}
Therefore, using the above inequality in the previous one in combination with $\sqrt{x+y}\leq \sqrt{x}+\sqrt{y}$ and $\sqrt{xy}\leq x+y$ leads to
\begin{align*}
  (p_h-\hp_h^{\,t})\left(\Vstar_{h+1}-\rV_{h+1}^{t}\right)(s,a) &\leq \frac{1}{H} \hp_h^{\,t} (\Vstar_{h+1}-\rV_{h+1}^{t})(s,a) + \left(4H^2 +\sqrt{8}H+\frac{2}{3}H\right) \frac{\beta(n_h^t(s,a),\delta)}{n_h^t(s,a)} \\
  &\leq  \frac{1}{H} \hp_h^{\,t} \left(\Vstar_{h+1}-\rV_{h+1}^{t}\right)(s,a) + 8H^2 \frac{\beta(n_h^t(s,a),\delta)}{n_h^t(s,a)}\cdot
\end{align*}
Going back to~\eqref{eq:last_term_gap_BPI} and applying just derived bounds and combining them with Lemma~\ref{lem:lower_bound_rQ_BPI}, we get
\begin{align*}
  \hp_h^{\,t} \utV_{h+1}^{t}(s,a) -p_h \rV_{h+1}^{t}(s,a) &\leq \hp_h^{\,t}\left(\utV_{h+1}^t -\rV_{h+1}^{t}\right)(s,a) \\
  &\qquad +3\sqrt{\Var_{\hp_h^{\,t}}(\utV_{h+1}^t)(s,a)\frac{\betastar(n_h^t(s,a),\delta) }{n_h^t(s,a)}} + 14 H^2 \frac{\beta(n_h^t(s,a),\delta)}{n_h^t(s,a)}\\
  &\qquad+\frac{1}{H} \hp_h^{\,t}(\utV_{h+1}^t-\ltV_{h+1}^t)(s,a)\\
  &\qquad+ \frac{1}{H} \hp_h^{\,t} (\Vstar_{h+1}-\rV_{h+1}^{t})(s,a) + 8H^2 \frac{\beta(n_h^t(s,a),\delta)}{n_h^t(s,a)}\\
  &\leq 3\sqrt{\Var_{\hp_h^{\,t}}(\utV_{h+1}^t)(s,a)\frac{\betastar(n_h^t(s,a),\delta) }{n_h^t(s,a)}} + 22 H^2 \frac{\beta(n_h^t(s,a),\delta)}{n_h^t(s,a)} \\
  &\qquad+\left(1+\frac{2}{H}\right) \hp_h^{\,t} (\utV_{h+1}^t-\rV_{h+1}^{t})(s,a)\,.
\end{align*}
Now, using the induction hypothesis and Lemma~\ref{lem:lower_bound_rQ_BPI} again, we obtain
\begin{align*}
\utQ_h^t(s,a)-\rQ^{t}_h(s,a) & \leq 6\sqrt{\Var_{\hp_h^{\,t}}(\utV_{h+1}^t)(s,a)\frac{\betastar(n_h^t(s,a),\delta) }{n_h^t(s,a)}} + 36 H^2 \frac{\beta(n_h^t(s,a),\delta)}{n_h^t(s,a)} \\
&\qquad +\left(1+\frac{3}{H}\right) \hp_h^{\,t} (\utV_{h+1}^t-\rV_{h+1}^{t})(s,a)\\
&\leq 6\sqrt{\Var_{\hp_h^{\,t}}(\utV_{h+1}^t)(s,a)\frac{\betastar(n_h^t(s,a),\delta) }{n_h^t(s,a)}} + 36 H^2 \frac{\beta(n_h^t(s,a),\delta)}{n_h^t(s,a)} \\
&\qquad +\left(1+\frac{3}{H}\right) \hp_h^{\,t} \pi^{t+1}_{h+1} G_{h+1}^t (s,a)\,.
\end{align*}

\paragraph{Step 2: Second case} In the alternative case, we have that
\begin{align*}
\rQ_h^t(s,a) & = \max\Bigg(0, r_h(s,a)  - 3\sqrt{\Var_{\hp_h^{\,t}}(\utV_{h+1}^t)(s,a) \frac{\betastar(n_h^t(s,a),\delta)}{n_h^t(s,a)}} -14H^2 \frac{\beta(n_h^t(s,a),\delta)}{n_h^t(s,a)}\\
&\qquad-\frac{1}{H}\hp_h^{\,t} (\utV_{h+1}^{t}-\ltV_{h+1}^t)(s,a)   +\hp_h^{\,t} \rV_{h+1}^{t}(s,a)\Bigg),
\end{align*}
and we use Lemma~\ref{lem:lower_bound_rQ_BPI} and the induction hypothesis to get
\begin{align*}
  \utQ_h^t(s,a)-\rQ^{t}_h(s,a) &\leq 6\sqrt{\Var_{\hp_h^{\,t}}(\utV_{h+1}^t)(s,a) \frac{\betastar(n_h^t(s,a),\delta)}{n_h^t(s,a)}}+24H^2 \frac{\beta(n_h^t(s,a),\delta)}{n_h^t(s,a)}\\ &\qquad+\frac{2}{H}\hp_h^{\,t} (\utV_{h+1}^{t}-\ltV_{h+1}^t)(s,a)   +\hp_h^{\,t} (\utV_{h+1}^{t}-\rV_{h+1}^{t})(s,a) \\
  &\leq 6\sqrt{\Var_{\hp_h^{\,t}}(\utV_{h+1}^t)(s,a) \frac{\betastar(n_h^t(s,a),\delta)}{n_h^t(s,a)}}+24H^2 \frac{\beta(n_h^t(s,a),\delta)}{n_h^t(s,a)}\\ &\qquad+\left(1+\frac{2}{H}\right)\hp_h^{\,t} (\utV_{h+1}^{t}-\rV_{h+1}^t)(s,a)\\
  &\leq 6\sqrt{\Var_{\hp_h^{\,t}}(\utV_{h+1}^t)(s,a) \frac{\betastar(n_h^t(s,a),\delta)}{n_h^t(s,a)}}+24H^2 \frac{\beta(n_h^t(s,a),\delta)}{n_h^t(s,a)}\\ &\qquad+\left(1+\frac{2}{H}\right)\hp_h^{\,t} \pi^{t+1}_{h+1} G_{h+1}^t(s,a)\,.
\end{align*}
\textbf{Conclusion} In both cases, since we assumed that $G_h^{\,t}(s,a)<H$ and using the definition of $G_h^{\,t}$, we proved that
\[\utQ_h^t(s,a)-\rQ^{t}_h(s,a) \leq G_{h}^t(s,a)\,.\]
In particular, for $h=1$ we conclude that
\[
\utV_1^t(s_1) - \rV_1^t(s_1) = \pi_1^{t+1}(\utQ_1^t-\rQ_1^t)(s_1) \leq \pi_1^{t+1} G_{1}^t(s_1).
\]

\end{proof}

\begin{lemma} \label{lem:lower_bound_rQ_BPI}
On  event $\cG$, we have that for all $(s,a,h)$,
\begin{align*}
  \rQ_h^t(s,a) &\leq \min\left(\ltQ_h^t(s,a),Q_h^{\pi^{t+1}}(s,a)\right)\\
  \rV_h^t(s) &\leq \min\left(\ltV_h^t(s),V_h^{\pi^{t+1}}(s) \right),
\end{align*}
where $\rQ_h^t$ and $\rV_h^t$ are defined in the proof of Lemma~\ref{lem:control_gap_BPI}.
\end{lemma}
\begin{proof}
We proceed by induction on $h$. For $h = H +1$ the inequalities are trivially true. For that induction step, we assume them true for $h+1$ and therefore for all $(s,a),$ we have
\begin{align*}
  \rQ_h^t(s,a) &\leq r_h(s,a) + p_h \rV_{h+1}^t(s,a)\\
   &\leq r_h(s,a) + p_h V_{h+1}^{\pi^{t+1}}(s,a) = Q_h^{\pi^{t+1}}(s,a)\,.
\end{align*}
Similarly to the above, we obtain
\begin{align*}
    \rQ_h^t(s,a) &\leq \max\Bigg(\!0, r_h(s,a)  - 3\sqrt{\Var_{\hp_h^{\,t}}(\utV_{h+1}^t)(s,a) \frac{\betastar(n_h^t(s,a),\delta)}{n_h^t(s,a)}}\\
    &\qquad-14H^2 \frac{\beta(n_h^t(s,a),\delta)}{n_h^t(s,a)}-\frac{1}{H}\hp_h^{\,t} (\utV_{h+1}^{t}-\ltV_{h+1}^t)(s,a)   +\hp_h^{\,t} \rV_{h+1}^{t}(s,a)\!\Bigg)\\
    &\leq \max\Bigg(\!0, r_h(s,a)  - 3\sqrt{\Var_{\hp_h^{\,t}}(\utV_{h+1}^t)(s,a) \frac{\betastar(n_h^t(s,a),\delta)}{n_h^t(s,a)}}\\
    &\qquad-14H^2 \frac{\beta(n_h^t(s,a),\delta)}{n_h^t(s,a)}-\frac{1}{H}\hp_h^{\,t} (\utV_{h+1}^{t}-\ltV_{h+1}^t)(s,a)   +\hp_h^{\,t} \ltV_{h+1}^{t}(s,a)\!\Bigg) = \ltQ_h^t(s,a).
\end{align*}
It remains to prove the inequalities for the value functions. Using the inequalities with the Q-value functions we get
\begin{align*}
\rV_h^t(s) &\leq \pi_h^{t+1} Q_h^{\pi^{t+1}}(s) = V_h^{\pi^{t+1}}(s), \quad \text{and} \\
 \rV_h^t(s) &\leq \pi_h^{t+1} \ltQ_h^{t}(s) \leq \max_{a\in\cA}  \ltQ_h^{t}(s,a) = \ltV_h^t(s)\,.
\end{align*}
\end{proof}

%

\theoremBPI*
\begin{proof}

  We first prove that \OurAlgorithmBPI is $(\epsilon,\delta)$-PAC. This is a simple consequence of Lemma~\ref{lem:control_gap_BPI}. Indeed, if our algorithm stops at time $\tau$ we know that on  event $\cG$
\[
    V_1^{\hpi}(s_1) = V_1^{\pi^{\tau+1}}(s_1) \geq \Vstar_1(s_1) -\pi_1^{\tau+1} G_1^{\tau}(s_1) \geq \Vstar_1(s_1) -\epsilon\,.
 \]
  We can then conclude the first part of the theorem by noting that by Lemma~\ref{lem:proba_master_event}, $\P(\cG)\geq 1-\delta$.

\medskip
\noindent
It remains to prove the bound on the sample complexity. In the rest of the proof we assume that the event $\cG$ holds and fix a round $T<\tau$.

\paragraph{Step 1: Upper bound on $G_1^t$} We first provide an upper bound on $G_h^{\,t}(s,a)$ for all $(s,a,h)$ and $t\leq T$. If $n_h^t(s,a)>0$ by definition of $G_h^{\,t}$ we have
\begin{align}
  G_h^{\,t}(s,a) &\leq  6\sqrt{\Var_{\hp_h^{\,t}}(\utV_{h+1}^t)(s,a) \frac{\betastar(n_h^t(s,a),\delta)}{n_h^t(s,a)}}+ 36 H^2 \frac{\beta(n_h^t(s,a),\delta)}{n_h^t(s,a)} \nonumber \\
  &\quad  + \left(1+\frac{3}{H}\right)\hp_h^{\,t} \pi_{h+1}^{t+1} G_{h+1}^t(s,a).
  \label{eq:def_G_without_min}
\end{align}
Now, we replace the empirical transition probability by the true one. Using Lemma~\ref{lem:Bernstein_via_kl} and that $0\leq G_h^{\,t}\leq H,$ we get
\begin{align*}
  (\hp_h^{\,t}-p_h) \pi_{h+1}^{t+1} G_{h+1}^t(s,a) &\leq \sqrt{2\Var_{p_h}(\pi_{h+1}^{t+1} G_{h+1}^t)(s,a) \frac{\beta(n_h^t(s,a),\delta)}{n_h^t(s,a)}} + \frac{2}{3}H \frac{\beta(n_h^t(s,a),\delta)}{n_h^t(s,a)}\\
  &\leq \frac{1}{H} p_h \pi_{h+1}^{t+1} G_{h+1}^t(s,a) +  3 H^2\frac{\beta(n_h^t(s,a),\delta)}{n_h^t(s,a)}\CommaBin
\end{align*}
where in the last line we used $\Var_{p_h}(\pi_{h+1}^{t+1} G_{h+1}^t)(s,a) \leq H \pi_{h+1}^{t+1} G_{h+1}^t(s,a)$ and $\sqrt{xy}\leq x+y$ for all $x,y\geq 0$. We also need to replace the variance of the upper confidence bound under the empirical transition by the variance of the optimal value function under the true transition probability.
Using Lemma~\ref{lem:switch_variance_bis} then Lemma~\ref{lem:switch_variance}, we obtain
\begin{align*}
\Var_{\hp_h^{\,t}}(\utV_{h+1}^t)(s,a) &\leq 2\Var_{p_h}(\utV_{h+1}^t)(s,a) + 4 H^2 \frac{\beta(n_h^t(s,a),\delta)}{n_h^t(s,a)}\\
&\leq 4 \Var_{p_h}(V_{h+1}^{\pi^{t+1}})(s,a) + 4 H p_h (\utV_{h+1}^t-V_{h+1}^{\pi^{t+1}})(s,a)+4H^2 \frac{\beta(n_h^t(s,a),\delta)}{n_h^t(s,a)}\\
&\leq 4 \Var_{p_h}(V_{h+1}^{\pi^{t+1}})(s,a) + 4 H p_h \pi_{h+1}^{t+1}G_{h+1}^{t}(s,a)+4H^2 \frac{\beta(n_h^t(s,a),\delta)}{n_h^t(s,a)}
\end{align*}
where we used \eqref{eq:G_lager_gap_bounds} from the proof of Lemma~\ref{lem:control_gap_BPI} in the last inequality. Next, using $\sqrt{x+y}\leq \sqrt{x}+\sqrt{y}$, $\sqrt{xy}\leq x+y,$ and $\betastar(n, \delta)\leq \beta(n,\delta)$ leads to
\begin{align*}
\sqrt{\Var_{\hp_h^{\,t}}(\utV_{h+1}^t)(s,a) \frac{\betastar(n_h^t(s,a),\delta)}{n_h^t(s,a)}} &\leq 2\sqrt{ \Var_{p_h}(V_{h+1}^{\pi^{t+1}})(s,a) \frac{\betastar(n_h^t(s,a),\delta)}{n_h^t(s,a)}}+ (2 H +4H^2) \frac{\beta(n_h^t(s,a),\delta)}{n_h^t(s,a)}\\
&\qquad+ \frac{1}{H} p_h \pi_{h+1}^{t+1}G_{h+1}^{t}(s,a)\\
& \leq 2\sqrt{ \Var_{p_h}(V_{h+1}^{\pi^{t+1}})(s,a) \frac{\betastar(n_h^t(s,a),\delta)}{n_h^t(s,a)}}+ 6H^2 \frac{\beta(n_h^t(s,a),\delta)}{n_h^t(s,a)}\\
&+ \frac{1}{H} p_h \pi_{h+1}^{t+1} G_{h+1}^t(s,a).
\end{align*}
Combining these two inequalities with~\eqref{eq:def_G_without_min} yields
\begin{align*}
  G_h^{\,t}(s,a) &\leq 12\sqrt{ \Var_{p_h}(V_{h+1}^{\pi^{t+1}})(s,a) \frac{\betastar(n_h^t(s,a),\delta)}{n_h^t(s,a)}}+ 36H^2 \frac{\beta(n_h^t(s,a),\delta)}{n_h^t(s,a)}\\
  &\qquad+ \frac{6}{H} p_h \pi_{h+1}^{t+1} G_{h+1}^t(s,a) + 36 H^2 \frac{\beta(n_h^t(s,a),\delta)}{n_h^t(s,a)}\\
  &\qquad+\left(1+\frac{3}{H}\right) \frac{1}{H} p_h \pi_{h+1}^{t+1} G_{h+1}^t + \left(1+\frac{3}{H}\right) 3 H^2\frac{\beta(n_h^t(s,a),\delta)}{n_h^t(s,a)}\\
  &\qquad+  \left(1+\frac{3}{H}\right)p_h \pi_{h+1}^{t+1} G_{h+1}^t(s,a)\\
  &\leq 12 \sqrt{ \Var_{p_h}(V_{h+1}^{\pi^{t+1}})(s,a) \frac{\betastar(n_h^t(s,a),\delta)}{n_h^t(s,a)}} + 84 H^2 \frac{\beta(n_h^t(s,a),\delta)}{n_h^t(s,a)}\\
  &\qquad +  \left(1+\frac{13}{H}\right)p_h \pi_{h+1}^{t+1} G_{h+1}^t(s,a).
\end{align*}
Since by construction, $G_h^{\,t}(s,a)\leq H,$ we have that for all $n_h^t(s,a)\geq 0,$
\begin{align*}
  G_h^{\,t}(s,a) & \leq 12 \sqrt{ \Var_{p_h}(V_{h+1}^{\pi^{t+1}})(s,a) \left(\frac{\betastar(n_h^t(s,a),\delta)}{n_h^t(s,a)} \wedge 1 \right)} + 84 H^2 \left(\frac{\beta(n_h^t(s,a),\delta)}{n_h^t(s,a)}\wedge 1\right)\\
  &\qquad +  \left(1+\frac{13}{H}\right)p_h \pi_{h+1}^{t+1} G_{h+1}^t(s,a).
\end{align*}
Unfolding the previous inequality and using $(1+13/H)^H\leq e^{13}$ we get
\begin{align*}
  \pi_1 G_1^t(s_1) & \leq 12e^{13} \sum_{h=1}^H \sum_{s,a} p_h^{t+1}(s,a) \sqrt{ \Var_{p_h}(V_{h+1}^{\pi^{t+1}})(s,a) \left(\frac{\betastar(n_h^t(s,a),\delta)}{n_h^t(s,a)} \wedge 1 \right)}\\
   &\qquad + 84  e^{13} H^2 \sum_{h=1}^H \sum_{s,a} p_h^{t+1}(s,a) \left(\frac{\beta(n_h^t(s,a),\delta)}{n_h^t(s,a)}\wedge 1\right),
\end{align*}
where we recall that $p_h^{t+1}(s,a)$ is the probability of reaching a state-action pair $(s,a)$ at step $h$ under  policy $\pi^{t+1}$. Using Lemma~\ref{lem:cnt_pseudo} we replace the counts by the pseudo-counts in the previous inequality as
\begin{align}
  \pi_1 G_1^t(s_1) & \leq 24e^{13} \sum_{h=1}^H \sum_{s,a} p_h^{t+1}(s,a) \sqrt{ \Var_{p_h}(V_{h+1}^{\pi^{t+1}})(s,a) \frac{\betastar(\bn_h^t(s,a),\delta)}{\bn_h^t(s,a)\vee 1} }\nonumber\\
   &\qquad + 336 e^{13} H^2 \sum_{h=1}^H \sum_{s,a} p_h^{t+1}(s,a) \frac{\beta(\bn_h^t(s,a),\delta)}{\bn_h^t(s,a)\vee 1  }\cdot \label{eq:upper_bound_G1}
\end{align}

\paragraph{Step 2: Law of total variance.} Using Lemma~\ref{lem:law_of_total_variance} we further upper-bound the first sum in~\eqref{eq:upper_bound_G1}. In particular, by Cauchy-Schwarz inequality, we obtain
\begin{align*}
  \sum_{h=1}^H \sum_{s,a} p_h^{t+1}(s,a) &\sqrt{ \Var_{p_h}(V_{h+1}^{\pi^{t+1}})(s,a) \frac{\betastar(\bn_h^t(s,a),\delta)}{\bn_h^t(s,a)\vee 1} }  \\
  &\leq
  \sqrt{  \sum_{h=1}^H \sum_{s,a} p_h^{t+1}(s,a) \Var_{p_h}(V_{h+1}^{\pi^{t+1}})(s,a) } \sqrt{  \sum_{h=1}^H \sum_{s,a} p_h^{t+1}(s,a) \frac{\betastar(\bn_h^t(s,a),\delta)}{\bn_h^t(s,a)\vee 1} }\\
  &\leq \sqrt{ \E_{\pi^{t+1}}\!\left[ \left(\sum_{h=1}^H r_{h}( s_{h},a_{h}) - V_1^{\pi^{t+1}}(s_1)\right)^2 \right] } \sqrt{  \sum_{h=1}^H \sum_{s,a} p_h^{t+1}(s,a) \frac{\betastar(\bn_h^t(s,a),\delta)}{\bn_h^t(s,a)\vee 1} }\\
  &\leq H  \sqrt{  \sum_{h=1}^H \sum_{s,a} p_h^{t+1}(s,a) \frac{\betastar(\bn_h^t(s,a),\delta)}{\bn_h^t(s,a)\vee 1} }\cdot
\end{align*}
Therefore, going back to~\eqref{eq:upper_bound_G1}, we obtain
\begin{align}
  \pi_1 G_1^t(s_1) & \leq 24e^{13} H  \sqrt{  \sum_{h=1}^H \sum_{s,a} p_h^{t+1}(s,a) \frac{\betastar(\bn_h^t(s,a),\delta)}{\bn_h^t(s,a)\vee 1} } \nonumber \\ &\quad + 336 e^{13} H^2 \sum_{h=1}^H \sum_{s,a} p_h^{t+1}(s,a) \frac{\beta(\bn_h^t(s,a),\delta)}{\bn_h^t(s,a)\vee 1}\cdot\label{eq:bound_G1_with_H}
\end{align}

\paragraph{Step 3: Summing over $t\leq T$} Since $T<\tau,$ we know that for $t\leq T,$ using due to the design of the stopping rule,
\[
 \epsilon \leq \pi_1^{t+1} G_1^t(s_1)\,.
\]
Therefore, summing the previous inequalities in combination with \eqref{eq:bound_G1_with_H}, yields
\begin{align*}
  (T+1) \epsilon &\leq 24e^{13} H \sum_{t=1}^T \sqrt{  \sum_{h=1}^H \sum_{s,a} p_h^{t+1}(s,a) \frac{\betastar(\bn_h^t(s,a),\delta)}{\bn_h^t(s,a)\vee 1} }
  + 336 e^{13} H^2 \sum_{t=1}^T \sum_{h=1}^H \sum_{s,a} p_h^{t+1}(s,a) \frac{\beta(\bn_h^t(s,a),\delta)}{\bn_h^t(s,a)\vee 1}\\
  &\leq 24e^{13} H\sqrt{T} \sqrt{ \sum_{t=1}^T \sum_{h=1}^H \sum_{s,a} p_h^{t+1}(s,a) \frac{\betastar(\bn_h^t(s,a),\delta)}{\bn_h^t(s,a)\vee 1} } + 336 e^{13} H^2 \sum_{t=1}^T \sum_{h=1}^H \sum_{s,a} p_h^{t+1}(s,a) \frac{\beta(\bn_h^t(s,a),\delta)}{\bn_h^t(s,a)\vee 1}\cdot
\end{align*}
Using successively that $\beta(.,\delta)$ and $\betastar(.,\delta)$ are increasing, and then Lemma~\ref{lem:sum_1_over_n}, we have
\begin{align*}
  \sum_{t=1}^T \sum_{h=1}^H \sum_{s,a} p_h^{t+1}(s,a) \frac{\betastar(\bn_h^t(s,a),\delta)}{\bn_h^t(s,a)\vee 1} &\leq \betastar(T,\delta) \sum_{t=1}^T \sum_{h=1}^H \sum_{s,a} p_h^{t+1}(s,a) \frac{1}{\bn_h^t(s,a)\vee 1} \\
  &\leq \betastar(T,\delta) \sum_{t=1}^T \sum_{h=1}^H \sum_{s,a} \frac{\bn_h^{t+1}(s,a)-\bn_h^t(s,a)}{\bn_h^t(s,a)\vee 1}\\
  &\leq 4 H S A \betastar(T,\delta) \log(T+2).
\end{align*}
Similarly, for the second sum we get
\begin{align*}
  \sum_{t=1}^T \sum_{h=1}^H \sum_{s,a} p_h^{t+1}(s,a) \frac{\beta(\bn_h^t(s,a),\delta)}{\bn_h^t(s,a)\vee 1} &\leq 4 H S A \beta(T,\delta) \log(T+2).
\end{align*}
Therefore, we obtain that
\begin{align*}
  (T+1) \epsilon &\leq  48e^{13} \sqrt{T}\sqrt{H^3 S A \betastar(T,\delta) \log(T+2)} + 1344 e^{13} H^3 S A \beta(T,\delta) \log(T+2).
\end{align*}
We assume that $\tau>0$ otherwise the result is trivially true. Since the above inequality is true for all $T<\tau,$ we get the functional inequality for $\tau$
\begin{equation}
  \label{eq:bound_tau_functional_BPI}
  \epsilon \tau \leq 48e^{13} \sqrt{ \tau H^3 S A \betastar(\tau -1,\delta) \log(\tau+1)} + 1344 e^{13} H^3 S A \beta(\tau-1,\delta) \log(\tau+1).
\end{equation}

\paragraph{Step 3: Upper bound on $\tau$} It remains to invert~\eqref{eq:bound_tau_functional_BPI}. Using the definition of $\beta$ and $\betastar$ yields
\begin{align*}
  \epsilon \tau &\leq 48e^{13} \sqrt{ \tau H^3 S A  \big( \log(3SAH/\delta) \log(8e\tau) + \log(8e\tau)^2\big)} \\
  &\qquad + 1344 e^{13} H^3 S A  \big( \log(3SAH/\delta) \log(8e\tau) + S\log(8e\tau)^2\big).
\end{align*}
We can now use Lemma~\ref{lem:reverse_inequality} with
\[ C = 48e^{13} \sqrt{H^3SA}/\epsilon,\, A = \log(3SAH/\delta) ,\, B = 1,\,E= S,\, D = 1344 e^{13} H^3 S A /\epsilon, \, \text{\ and\ } \alpha = 8e \]
to get
\[
\tau \leq \frac{H^3SA}{\epsilon^2} \big( \log(3SAH/\delta) +S\big) C_1+1,
\]
with $C_1 \triangleq 5904e^{26} \log\left(e^{30} \left( \log(3SAH/\delta) +S\right) H^3SA/\epsilon\right)^2$ and used that $\epsilon\leq 1/S^2$.
\end{proof}

\newpage
\section{Technical lemmas}

\subsection{A Bellman-type equation for the variance}
\label{app:Bellman_variance}
For a deterministic policy $\pi$ we define Bellman-type equations for the variances as follows
\begin{align*}
  \Qvar_h^\pi(s,a) &\triangleq \Var_{p_h}{V_{h+1}^\pi}(s,a) + p_h \Vvar^\pi_{h+1}(s,a)\\
  \Vvar_h^\pi(s) &\triangleq \Qvar^\pi_h (s, \pi(s))\\
  \Vvar_{H+1}^\pi(s)&\triangleq0,
\end{align*}
where $\Var_{p_h}(f)(s,a) \triangleq \E_{s' \sim p_h(\cdot | s, a)} \big[(f(s')-p_h f(s,a))^2\big]$ denotes the \emph{variance operator.}
 In particular, the function $s \mapsto \Vvar_1^\pi(s)$ represents the average sum of the local variances $\Var_{p_h}{V_{h+1}^\pi}(s,a)$ over a trajectory following the policy $\pi$, starting from $(s, a)$. Indeed, the definition above implies that
 \[\Vvar_1^\pi(s_1) = \sum_{h=1}^H\sum_{s,a} p_h^\pi(s,a) \Var_{p_h}(V_{h+1}^\pi)(s,a).
 \]
 The lemma below shows that we can relate the global variance of the cumulative reward over a trajectory to the average sum of local variances.
\begin{lemma}[Law of total variance]  For any deterministic policy $\pi$ and for all $h\in[H]$,
  \[
  \E_\pi\!\left[  \left(\sum_{h'=h}^H r_{h'}( s_{h'},a_{h'}) - Q_h^\pi(s_h,a_h)\right)^{\!\!2}\middle| (s_h,a_h)=(s,a) \right] = \Qvar_h^\pi(s,a).
  \]
In particular,
\[
\E_\pi\!\left[ \left(\sum_{h=1}^H r_{h}( s_{h},a_{h}) - V_1^\pi(s_1)\right)^{\!\!2} \right] = \Vvar_1^\pi(s_1) = \sum_{h=1}^H\sum_{s,a} p_h^\pi(s,a) \Var_{p_h}(V_{h+1}^\pi)(s,a).
\]
\label{lem:law_of_total_variance}
\end{lemma}
\begin{proof}
	We proceed by induction. The statement in Lemma~\ref{lem:law_of_total_variance} is trivial for $h=H+1$. We now assume that it holds for $h+1$ and prove that it also holds for $h$. For this purpose, we compute
	\begin{align*}
		&
		\E_\pi\left[\!
		 \left(\sum_{h'=h}^H r_{h'}( s_{h'},a_{h'}) - Q_h^\pi(s_h,a_h)\right)^{\!\!2}
		 \middle| (s_h,a_h) \right] \\
		& =
		\E_\pi\left[\! \left( Q_{h+1}^\pi(s_{h+1},a_{h+1}) - p_h V_{h+1}^\pi(s_h,a_h) + \sum_{h'=h+1}^H r_{h'}( s_{h'},a_{h'}) - Q_{h+1}^\pi(s_{h+1},a_{h+1})\right)^{\!\!2}
		\middle| (s_h,a_h) \right] \\
		& =
		\E_\pi\left[\!
			\left( Q_{h+1}^\pi(s_{h+1},a_{h+1}) - p_h V_{h+1}^\pi(s_h,a_h) \right)^{\!\!2}
		\middle| (s_h,a_h) \right] \\
		&
		+ \E_\pi\left[\!
		\left( \sum_{h'=h+1}^H r_{h'}( s_{h'},a_{h'}) - Q_{h+1}^\pi(s_{h+1},a_{h+1}) \right)^{\!\!2}
		\middle| (s_h,a_h) \right] \\
		& + 2 \E_\pi\left[\!
			\left( \sum_{h'=h+1}^H r_{h'}( s_{h'},a_{h'}) - Q_{h+1}^\pi(s_{h+1},a_{h+1}) \right)
			\left( Q_{h+1}^\pi(s_{h+1},a_{h+1}) - p_h V_{h+1}^\pi(s_h,a_h) \right)
		\middle| (s_h,a_h) \right].
	\end{align*}
	The definition of  $Q_{h+1}^\pi(s_{h+1},a_{h+1})$ implies that \[\E_\pi\!\left[ \sum_{h'=h+1}^H r_{h'}( s_{h'},a_{h'}) - Q_{h+1}^\pi(s_{h+1},a_{h+1})	\middle| (s_{h+1},a_{h+1}) \right] = 0.\]
	Therefore, the law of total expectation gives us
	\begin{align*}
		&\E_\pi\left[\!
		 \left(\sum_{h'=h}^H r_{h'}( s_{h'},a_{h'}) - Q_h^\pi(s_h,a_h)\right)^{\!\!2}
		 \middle| (s_h,a_h) \right] \\ & = \E_\pi\!\left[
		\left( V_{h+1}^\pi(s_{h+1}) - p_h V_{h+1}^\pi(s_h,a_h) \right)^2
		\middle| (s_h,a_h) \right]
		+ \E_\pi\!\left[
		\left( \sum_{h'=h+1}^H \!\! \!r_{h'}( s_{h'},a_{h'}) - Q_{h+1}^\pi(s_{h+1},a_{h+1}) \right)^2
		\middle| (s_h,a_h) \right] \\
		& = \Var_{p_h}{V_{h+1}^\pi}(s_h,a_h) \\ & \hspace{0.3cm} + \!\!\!\!\!\!\sum_{(s_{h+1},a_{h+1})}\!\!\!\!\!\! p_h(s_{h+1} | s_h,a_h)\ind_{\left(a_{h+1} = \pi(s_{h+1})\right)} \E_\pi\left[\left( \sum_{h'=h+1}^H \!\!\! r_{h'}( s_{h'},a_{h'}) - Q_{h+1}^\pi(s_{h+1},a_{h+1}) \right)^2
		\middle| (s_{h+1},a_{h+1})\right] \\
		&  = \Var_{p_h}{V_{h+1}^\pi}(s_h,a_h) + p_h \Vvar^\pi_{h+1}(s_h,a_h) \\
		& = \sigma Q_h^\pi(s_h,a_h)
	\end{align*}
	where in the third equality we used the inductive hypothesis and the definition of $\sigma V_{h+1}^{\pi}$. 
\end{proof}

\subsection{Counts to pseudo-counts}
\label{app:count}
\begin{lemma}\label{lem:cnt_pseudo} On event $\cE^{\text{\normalfont cnt}}$, $\forall  h \in [H], (s,a) \in \cS \times \cA$,
\[ \forall t \in \N^\star, \ \frac{\beta(n_h^t(s,a),\delta)}{n_h^t(s,a)}\wedge 1 \leq 4 \frac{\beta(\bar n_h^t(s,a),\delta)}{\bar n_h^t(s,a)\vee 1}\cdot\]
\end{lemma}
\begin{proof}
As event $\cE^{\mathrm{\normalfont cnt}}$ holds, we know that for all $t < \tau$,
\begin{align*}n_{\ell}^{t}(s,a) \geq \frac{1}{2}\bar n_{\ell}^{t}(s,a) - \beta^{\cnt}(\delta).
\end{align*}
We now distinguish two cases. First, if $\beta^{\cnt}(\delta) \leq \tfrac{1}{4}\bar n_{\ell}^{t}(s,a)$, then \[\frac {\beta(n_{\ell}^t(s,a), \delta)} {n_{\ell}^t(s,a)}\wedge 1 \leq \frac {\beta(n_{\ell}^t(s,a), \delta)} {n_{\ell}^t(s,a)}\leq \frac {\beta\left(\tfrac{1}{4}\bar n_{\ell}^{t}(s,a), \delta\right)} {\tfrac{1}{4}\bar n_{\ell}^{t}(s,a)} \leq 4   \frac {\beta\left(\bar n_{\ell}^{t}(s,a), \delta\right)} {\bar n_{\ell}^{t}(s,a) \vee 1}\CommaBin\]
where we used that $x \mapsto \beta(x,\delta)/x$ is non-increasing for $x\geq 1$,  $x \mapsto \beta(x,\delta)$ is non-decreasing, and $\beta^{\cnt}(\delta) \geq 1$. Second,
if $\beta^{\cnt}(\delta) > \tfrac{1}{4}\bar n_{\ell}^{t}(s,a)$, a simple derivation gives that
\[\frac {\beta(n_{\ell}^t(s,a), \delta)} {n_{\ell}^t(s,a)}\wedge 1\leq   1 < 4 \frac{\beta^{\cnt}(\delta)}{\bar n_{\ell}^{t}(s,a) \vee 1} \leq 4 \frac{\beta(\bar n_{\ell}^{t}(s,a), \delta)}{\bar n_{\ell}^{t}(s,a) \vee 1}\CommaBin\]
where we used that $1 \leq \beta^{\cnt}(\delta)\leq \beta(0,\delta)$ and $x \mapsto \beta(x,\delta)$ is non-decreasing.
\end{proof}

\begin{lemma}
	\label{lem:sum_1_over_n}
	 For $T\in\N^\star$ and $(u_t)_{t\in\N^\star},$ for a sequence where  $u_t\in[0,1]$ and $U_t \triangleq \sum_{l=1}^t u_\ell$, we get
	\[
		\sum_{t=0}^T \frac{u_{t+1}}{U_t\vee 1} \leq 4\log(U_{T+1}+1).
	\]
\end{lemma}
\begin{proof}
	Notice that	\begin{align*}
		\sum_{t=0}^T \frac{u_{t+1}}{U_t\vee 1} &\leq 4 \sum_{t=0}^T \frac{u_{t+1} }{2U_t + 2} \\
		&\leq  4\sum_{t=0}^T \frac{U_{t+1}-U_{t}}{U_{t+1} + 1}\\
		&\leq 4\sum_{t=0}^T \int_{U_t}^{U_{t+1}} \frac{1}{x+1} \mathrm{d}x\\
		& = 4\log(U_{T+1}+1).
	\end{align*}
\end{proof}

\subsection{On Bernstein inequality}
\label{app:Bernstein}
We reproduce a Bernstein-type inequality similar to the one by \citet{talebi2018variance}.
\begin{lemma}\label{lem:Bernstein_via_kl}
Let $p,q\in\Sigma_S,$ where $\Sigma_S$ denotes the probability simplex of dimension $S-1$. For all $\alpha> 0 $, for all functions $f$ defined on $\cS$ with $0\leq f(s) \leq b$, for all $s\in\cS$, if $\KL(p,q)\leq \alpha$ then
\begin{align*}
	|p f - q f |&\leq  \sqrt{2\Var_{q}(f)\alpha}+\frac{2}{3}b \alpha,
\end{align*}
where use the expectation operator defined as $pf \triangleq \E_{s\sim p} f(s)$ and the variance operator defined as
$\Var_p(f) \triangleq \E_{s\sim p} \big(f(s)-\E_{s'\sim p}f(s')\big)^2 = p(f-pf)^2.$
\end{lemma}
\begin{proof}
We only prove that
\begin{align*}
	q f - p f &\leq  \sqrt{2\Var_{q}(f)\alpha}+\frac{2}{3}b \alpha
\end{align*}
as an upper bound on $p f - q f$ can be obtained analogically. We assume that $q f > p f,$ otherwise the inequality is trivially true. Furthermore, without loss of generality, we consider that $0 \leq f(s) \leq 1$ for all $s\in\cS$: when $f$ is bounded by $b$, we can apply the inequality to $f/b$ and by homogeneity, we get the desired general version.
%
Using the variational formula for the Kullback-Leibler divergence, we have
\begin{align*}
	\KL(p,q) &= \sup_{g \in \R^S} p g - \log(q e^{g})\\
	&\geq \sup_{\lambda\geq 0} \lambda(q f-pf) - \log\left(q e^{\lambda(q f -f)}\right),
\end{align*}
where we chose $g \triangleq \lambda (q f -f)$ and used that $pqf = qf$ since $qf$ is a scalar. Given  that the mapping $u \to (e^u -u -1)/u^2$ is non-decreasing in $u$ and that $q f -f(s) \leq 1,$ we obtain
\[
e^{\lambda \big(q f -f(s)\big)} - \lambda\big( q f -f(s) \big) - 1 \leq \big(q f - f(s)\big)^2 (e^\lambda -\lambda-1).
\]
Furthermore, by taking the expectation of the previous inequality with respect to $q$, using the property $\log(1+x)\leq x,$ and define function $\phi(\lambda) \triangleq e^\lambda -\lambda -1,$ we get
\begin{align*}
	\log(q e^{\lambda(q f-f)}) \leq \log\big(1+  \Var_{q}(f) \phi(\lambda)\big)\leq  \Var_{q}(f) \phi(\lambda).
\end{align*}
Now using the previous inequality in the variational formula above with the fact that the convex conjugate of $\phi$ is for $u\geq 0$, $\sup_{\lambda\geq 0}\lambda u -\phi(\lambda) \triangleq h(u) \geq u^2/\left(2(1+u/3)\right)$, yields
\begin{align*}
	\KL(p,q) &\geq \sup_{\lambda\geq 0} \lambda(q f-pf) -  \Var_{q}(f) \phi(\lambda)\\
	&=  \Var_{q}(f) h\left( \frac{q f-pf}{ \Var_{q}(f)}  \right)\\
	&\geq \frac{(q f-pf)^2}{2\left(\Var_{q}(f)+(q f-pf)/3\right)}\cdot
\end{align*}
Since by assumption $\KL(p,q) \leq \alpha$, we get
\[
2\alpha\left(\Var_{q}(f)+(q f-pf)/3\right)- (q f-pf)^2\geq 0.
\]
Hence $qf - pf$ is upper bounded by the positive root of the polynomial in the left hand side,
\[\frac{\alpha}{3} + \sqrt{2\alpha \Var_q(f) + \frac{\alpha^2}{9}} \leq \sqrt{2\alpha \Var_q(f)} + \frac{2}{3}\alpha,\]
using that $\sqrt{x+y} \leq \sqrt{x} + \sqrt{y}$.
\end{proof}

\begin{lemma}
\label{lem:switch_variance_bis}
Let $p,q\in\Sigma_S$ and  $f$ is a function defined on $\cS$ such that $0\leq f(s) \leq b$ for all $s\in\cS$. If $\KL(p,q)\leq \alpha$ then
\begin{align*}
  \Var_q(f) &\leq 2\Var_p(f) +4b^2 \alpha \quad \text{and}\\
  \Var_p(f) &\leq 2\Var_q(f) +4b^2 \alpha.
\end{align*}
\end{lemma}
\begin{proof}
Let $\tp$ be the distribution of the pair of random variables $(X,Y)$ where $X,Y$ are i.i.d.\,according to the distribution $p$. Similarly, let $\tq$ be the distribution of the pair of random variables $(X,Y)$ where $X,Y$ are i.i.d.\,according to distribution $q$. Since Kullback–Leibler divergence is additive for independent distributions, we know that
\[\KL(\tp,\tq) = 2\KL(p,q) \leq 2\alpha.\]
Using Lemma~\ref{lem:Bernstein_via_kl} for the function $g(x,y) = (f(x)-f(y))^2$ defined on $\cS^2$, such that  $0\leq g\leq b^2,$ we get
\begin{align*}
  |\tp g- \tq g| &\leq \sqrt{4\Var_{\tq}(g) \alpha} + \frac{4}{3}b^2 \alpha\\
  &\leq  \sqrt{4 b^2 \alpha \tq g  } + \frac{4}{3}b^2 \alpha\\
  &\leq \frac{1}{2} \tq g + \frac{10}{3} b^2 \alpha\,,
\end{align*}
where in the last line we used $2\sqrt{xy}\leq x +y$ for $x,y\geq 0$. In particular we obtain
\begin{align*}
  \tp g &\leq \frac{3}{2} \tq g + \frac{10}{3} b^2\alpha\\
  \tq g &\leq 2 \tp g + \frac{20}{3}b^2 \alpha \,.
\end{align*}
To conclude, it remains to note that
\[ \tp g = 2\Var_p(f) \text{ and } \tq g = 2\Var_q(f). \]
\end{proof}

\begin{lemma}
	\label{lem:switch_variance}
	For $p,q\in\Sigma_S$, for $f,g$ two functions defined on $\cS$ such that $0\leq g(s),f(s) \leq b$ for all $s\in\cS$, we have that
	\begin{align*}
 \Var_p(f) &\leq 2 \Var_p(g) +2 b p|f-g|\quad\text{and} \\
 \Var_q(f) &\leq \Var_p(f) +3b^2\|p-q\|_1,
\end{align*}
where we denote the absolute operator by $|f|(s)= |f(s)|$ for all $s\in\cS$.
\end{lemma}
\begin{proof}
First note that
\[
\Var_p(f) = p(f-g+ g-p g + p g- p f)^2 \leq 2 p(f-g - p f + p g)^2 +2 p(g-p g)^2 = 2\Var_p(f-g)+2\Var_p(g).
\]
From the above we can immediately conclude the proof of the first inequality with
\[
\Var_p(f-g) \leq p(f-g)^2 \leq b p|f-g|,
\]
where we used that for all $s\in\cS$, $0\leq |f(s)-g(s)| \leq b$. For the second inequality, using the Hölder inequality,
\begin{align*}
	\Var_q(f) &= pf^2 - (pf)^2 +(q-p)f^2 + (pf)^2 -(qf)^2 \\
	&\leq \Var_p(f) + b^2\|p-q\|_1 +2b^2\|p-q\|_1 \\
	&\leq \Var_p(f) +3 b^2 \|p-q\|_1.
\end{align*}
\end{proof}

\subsection{An auxiliary inequality}
\begin{lemma}\label{lem:reverse_inequality}
Let $A,B,C,D,E,$ and $\alpha$ be positive scalars such that $1\leq B\leq E$ and $\alpha\geq e$. If $\tau\geq 0$ satisfies
\begin{equation}
  \label{eq:functionnal_ineq_lemma}
\tau \leq C\sqrt{\tau \big( A \log(\alpha \tau) +B \log(\alpha\tau)^2\big)}+ D \left( A\log(\alpha\tau)+E\log(\alpha\tau)^2\right)
\end{equation}
then
\[
\tau \leq C^2 \left(A+B\right) C_1^2 + \left(D+2\sqrt{D}C\right)\left(A+E\right)C_1^2+1,
\]
where
\[
C_1 = \frac{8}{5}\log\left( 11 \alpha^2 \left(A+E\right)\left(C+D\right)\right).
\]
\end{lemma}
\begin{proof}
We can assume that $\tau \geq 1$ otherwise the result is trivially true. Using $\log(x)\leq x^\beta /\beta$ for $x\geq 0, \beta >0,$ and $\alpha\geq 1$ we get
\begin{align*}
  \tau \left(A\log(\alpha \tau) + B\log(\alpha \tau )^2\right)&\leq \tau \left(4A\alpha^{1/4}\tau^{1/4}+B(8\alpha^{1/8}\tau^{1/8})^2 \right)\\
  &\leq (4A\alpha^{1/4}+64B\alpha^{1/4})\tau^{5/4}\\
  &\leq 64 \alpha^{1/4}(A+B)\tau^{5/4},
\end{align*}
and consequently,
\[
C\sqrt{\tau \left(A\log(\alpha\tau)+B\log(\alpha\tau)^2\right)}\leq 8C \alpha \sqrt{A+B} \tau^{5/8}.
\]
Similarly,
\begin{align*}
  D\left(A \log(\alpha \tau) + E\log(\alpha\tau)^2\right) &\leq D \left( 8A\alpha \tau^{5/8} +E \left(\frac{16}{5} \alpha^{5/16} \tau^{5/16}\right)^{\!\!2}\right)\\
  & \leq 11 D\alpha\left(A +E\right) \tau^{5/8}.
\end{align*}
Using the above two inequalities in combination with~\eqref{eq:functionnal_ineq_lemma} we obtain a first crude upper bound
\[
\tau \leq \left(11\alpha (A+E)(C+D)\right)^{8/5}.
\]
if we let $C_1$ be the constant defined as
\[
C_1 \triangleq \frac{8}{5}\log\left( 11 \alpha^2 (A+E)(C+D)\right),
\]
then upper-bounding the log in~\eqref{eq:functionnal_ineq_lemma} with $C_1$ yields a simpler inequality verified by $\tau,$
\[
\tau \leq C\sqrt{\tau \left( A C_1 +B C_1^2\right)}+D\left((A C_1+E C_1^2\right).
\]
Next, define the constants $a  \triangleq C \sqrt{A C_1 + B C_1^2}$ and $b \triangleq D(AC_1+BC_1^2)$. Hence, $x = \sqrt{\tau}$ satisfies
\[x^2 -a x -b \leq 0,\]
which implies that $x$ is upper-bounded by the larger roots of the polynomial above
\[
x \leq \frac{a + \sqrt{a^2+4b}}{2} \leq a+\sqrt{b}.
\]
We there deduce the following bound on $\tau$
\begin{align*}
  \tau &\leq (a+\sqrt{b})^2\\
  &\leq a^2 +2\sqrt{b}a +b\\
  &\leq C^2 \left(A+B\right) C_1^2 + \left(D+2\sqrt{D}C\right)\left(A+E\right)C_1^2,
\end{align*}
where in the last inequality we used  that $C_1 \leq C_1^2$ since $\alpha\geq e$ and $B\leq E$.
\end{proof}

\end{document}